%% file: main.tex
\documentclass{article}

\usepackage[nonatbib,final]{neurips_2021}
\usepackage[numbers,sort&compress]{natbib}

\input{macros}

\title{Curriculum Design for Teaching via Demonstrations: Theory and Applications}

\author{
  \textbf{Gaurav Yengera}\textsuperscript{$1$,$2$} \quad  \quad
  \textbf{Rati Devidze}\textsuperscript{$1$} \quad \ \ 
  \textbf{Parameswaran Kamalaruban}\textsuperscript{$3$} \quad   
  \textbf{Adish Singla}\textsuperscript{$1$}
  \\
  \small{gyengera@mpi-sws.org} \quad
  \small{rdevidze@mpi-sws.org} \quad  \quad 
  \small{kparameswaran@turing.ac.uk} \quad \ \ 
  \small{adishs@mpi-sws.org}
  \\
  \\
  \textsuperscript{$1$}Max Planck Institute for Software Systems (MPI-SWS), Saarbrucken, Germany\\
  \textsuperscript{$2$}Saarland University, Saarland Informatics Campus (SIC), Saarbrucken, Germany\\  
  \textsuperscript{$3$}The Alan Turing Institute, London, UK
}

\begin{document}

\maketitle

\newtoggle{longversion}
\settoggle{longversion}{true}

\vspace{-2mm}
\input{0_abstract}

\input{1_introduction}

\input{1.1_additional_related_work}

\input{2_problem_setup}

\input{3.1_curriculum_algorithms_new}

\input{3.2_curriculum_learning_theory}
\input{4_driving_experiments_new}

\input{5_co_experiments}
\input{6_conclusion}

\iftoggle{longversion}{
    \clearpage
    \onecolumn
    \appendix
    {
        \allowdisplaybreaks
        \input{9.0_appendix_table-of-contents}
        \input{9.1_appendix_proofs}
        \input{9.4_appendix_additional_experiment_details}
        
    }
}

\clearpage
\bibliographystyle{unsrt}
\bibliography{main}

\end{document}

%% file: macros.tex
\usepackage{times}
\usepackage{soul}
\usepackage{url}
\usepackage[utf8]{inputenc}
\usepackage[small]{caption}
\usepackage{subcaption}
\usepackage{graphicx}
\usepackage{amsmath}
\usepackage{amsfonts}
\usepackage{amsthm}
\usepackage{booktabs}
\usepackage{algorithm}
\usepackage[noend]{algpseudocode}
\usepackage{longtable}
\usepackage{multicol}
\usepackage{bbm}
\usepackage{bm}
\usepackage{siunitx}
\usepackage{mathtools}
\usepackage{etoolbox}
\usepackage[T1]{fontenc} 
\usepackage{enumitem}
\usepackage{multirow}
\usepackage{wrapfig}
\usepackage{nicefrac}

\newtheorem{theorem}{Theorem}
\newtheorem{definition}{Definition}
\newtheorem{lemma}{Lemma}

\newcommand{\E}{\mathbb{E}}

\newcommand{\w}{\ensuremath{\theta_t}}
\newcommand{\wnext}{\ensuremath{\theta_{t+1}}}
\newcommand{\wopt}{\ensuremath{\theta^*}}

\DeclareMathOperator*{\argmax}{arg\,max}
\newcommand\norm[1]{\left\lVert#1\right\rVert}

\newcommand{\bcc}[1]{\left\{{#1}\right\}}
\newcommand{\brr}[1]{\left({#1}\right)}
\newcommand{\bss}[1]{\left[{#1}\right]}
\newcommand{\ipp}[2]{\left\langle{#1},{#2}\right\rangle}

\newcommand{\Expectover}[2]{\mathbb{E}_{#1}\!\left[#2\right]}
\newcommand{\abs}[1]{\left\vert#1\right\vert}

\newcommand{\Prob}[1]{\mathbb{P}\bcc{{#1}}}

\usepackage[usenames,dvipsnames]{xcolor}
\usepackage{hyperref}
\hypersetup{ %
	pdfauthor={},
	pdftitle={},
	pdfsubject={},
    pdflang={en},
	bookmarks,
	bookmarksnumbered,
	backref=page,
	pageanchor=true,
	plainpages=false,
	colorlinks=true,%
	linktocpage=true,
	citecolor=MidnightBlue,%
	filecolor=BrickRed,%
	linkcolor=NavyBlue,%
	urlcolor=NavyBlue 
}

\newcommand{\irlmodel}{MaxEnt-IRL}
\newcommand{\bcmodel}{CrossEnt-BC}
\newcommand{\cur}{\textsc{Cur}}

\newcommand{\agn}{\textsc{Agn}}
\newcommand{\omn}{\textsc{Omn}}
\newcommand{\bbox}{\textsc{BBox}}
\newcommand{\scot}{\textsc{Scot}}
\newcommand{\curl}{\textsc{Cur-L}}
\newcommand{\curt}{\textsc{Cur-T}}
\newcommand{\startstate}{s_{t}}

%% file: 0_abstract.tex
\begin{abstract}
\looseness-1We consider the problem of teaching via demonstrations in sequential decision-making settings. In particular, we study how to design a personalized curriculum over demonstrations to speed up the learner's convergence.~We provide a unified curriculum strategy for two popular learner models: Maximum Causal Entropy Inverse Reinforcement Learning (\irlmodel) and Cross-Entropy Behavioral Cloning (\bcmodel). Our unified strategy induces a ranking over demonstrations based on a notion of difficulty scores computed w.r.t. the teacher's optimal policy and the learner's current policy. Compared to the state of the art, our strategy doesn't require access to the learner's internal dynamics and still enjoys similar convergence guarantees under mild technical conditions. Furthermore, we adapt our curriculum strategy to the setting where no teacher agent is present using task-specific difficulty scores. Experiments on a synthetic car driving environment and navigation-based environments demonstrate the effectiveness of our curriculum strategy.
\end{abstract}

%% file: 1_introduction.tex
\section{Introduction}
\label{sec:introduction}

Imitation learning is a paradigm in which a learner acquires a new set of skills by imitating a teacher's behavior. The importance of imitation learning is realized in real-world applications where the desired behavior cannot be explicitly defined but can be demonstrated easily. These applications include the settings involving both human-to-machine interaction~\cite{Schaal1997,argall2009survey,kober2013reinforcement,cakmak2014eliciting}, and human-to-human interaction~\cite{buchsbaum2011children,shafto2014rational}. The two most popular approaches to imitation learning are Behavioral Cloning (BC)~\cite{bc1991} and Inverse Reinforcement Learning (IRL)~\cite{russell1998learning}. BC algorithms aim to directly match the behavior of the teacher using supervised learning methods. IRL algorithms operate in a two-step approach: first, a reward function explaining the teacher's behavior is inferred; then, the learner adopts a policy corresponding to the inferred reward. 

In the literature, imitation learning has been extensively studied from the learner's point of view to design efficient learning algorithms~\cite{abbeel2004,ziebart2008,boularias2011relative,wulfmeier2015maximum,finn2016guided,dorsaactive2017,osa2018algorithmic}. However, much less work is done from the teacher's point of view to reduce the number of demonstrations required to achieve the learning objective. In this paper, we focus on the problem of Teaching via Demonstrations (TvD), where a helpful teacher assists the imitation learner in converging quickly by designing a personalized curriculum~\cite{walsh2012dynamic,cakmak2012algorithmic,brown2019machine,lopes2009active,amin2017repeated}. Despite a substantial amount of work on curriculum design for reinforcement learning agents~\cite{svetlik2017automatic,florensa2017reverse,florensa2018automatic,czarnecki2018mix,narvekar2019learning,sukhbaatar2018intrinsic,riedmiller2018learning}, curriculum design for imitation learning agents is much less investigated. 

\looseness-1Prior work on curriculum design for IRL learners has focused on two concrete settings: non-interactive and interactive. In the non-interactive setting~\cite{cakmak2012algorithmic,brown2019machine}, the teacher provides a near-optimal set of demonstrations as a single batch. These curriculum strategies do not incorporate any feedback from the learner, hence unable to adapt the teaching to the learner's progress. In the interactive setting~\cite{imt_ijcai2019}, the teacher can leverage the learner's progress to adaptively choose the next demonstrations to accelerate the learning process. However, the existing state-of-the-art work~\cite{imt_ijcai2019} has proposed interactive curriculum algorithms that are based on learning dynamics of a specific IRL learner model (i.e., the learner's gradient update rule); see further discussion in Section~\ref{subsec:detailed-comparison}. In contrast, we focus on designing an interactive curriculum algorithm with theoretical guarantees that is agnostic to the learner's dynamics. This will enable the algorithm to be applicable for a broad range of learner models, and in practical settings where the learner's internal model is unknown (such as tutoring systems with human learners). A detailed comparison between our curriculum algorithm and the prior state-of-the-art algorithms from \cite{brown2019machine,imt_ijcai2019} is presented in Section \ref{subsec:detailed-comparison}.

Our approach is motivated by works on curriculum design for supervised learning and reinforcement learning algorithms that use a ranking over the training examples using a difficulty score~\cite{elman1993learning,bengio2009curriculum,zaremba2014learning,weinshall2018curriculum,weinshall2018theory,asada1996purposive,wu2016training}. In particular, our work is inspired by theoretical results on curriculum learning for linear regression models~\cite{weinshall2018curriculum}. We define difficulty scores for any demonstration based on the teacher's optimal policy and the learner's current policy. We then study the differential effect of the difficulty scores on the learning progress for two popular imitation learners: Maximum Causal Entropy Inverse Reinforcement Learning (\irlmodel)~\cite{ziebart2008} and Cross-Entropy loss-based Behavioral Cloning (\bcmodel)~\cite{bain1995framework}. Our main contributions are as follows:\footnote{Github repo: \url{https://github.com/adishs/neurips2021_curriculum-teaching-demonstrations_code}.}
\begin{enumerate}[parsep=0.5pt, leftmargin=*,labelindent=0.5pt]
    \item \looseness-1Our analysis for both \irlmodel~and \bcmodel~learners leads to a unified curriculum strategy, i.e., a preference ranking over demonstrations. This ranking is obtained based on the ratio between the demonstration's likelihood under the teacher's optimal policy and the learner's current policy. Experiments on a synthetic car driving environment validate our curriculum strategy.
    \item \looseness-1For the {\irlmodel} learner, we prove that our curriculum strategy achieves a linear convergence rate (under certain mild technical conditions), notably without requiring access to the learner's dynamics.
    \item \looseness-1We adapt our curriculum strategy to the learner-centric setting where a teacher agent is not present through the use of task-specific difficulty scores. As a proof of concept, we show that our strategy accelerates the learning process in synthetic navigation-based environments.
\end{enumerate}

\subsection{Comparison to Existing Approaches on Curriculum Design for Imitation Learning}
\label{subsec:detailed-comparison}

\looseness-1In the non-interactive setting, \cite{brown2019machine} have proposed a batch teaching algorithm (\scot) by showing that the teaching problem can be formulated as a set cover problem. In contrast, our algorithm is interactive in nature and hence, can leverage the learner's progress (see experimental results in Section~\ref{sec:experiments}).

In the interactive setting, \cite{imt_ijcai2019} have proposed the Omniscient algorithm (\omn) based on the iterative machine teaching (IMT) framework~\cite{liu2017iterative}. Their algorithm obtains strong convergence guarantees for the \irlmodel~learner model; however, requires \emph{exact} knowledge of the learner's dynamics (i.e, the learner's update rule). Our algorithm on the other hand is agnostic to the learner's dynamics and is applicable to a broader family of learner models (see Sections \ref{sec:curr-teacher-theory} and \ref{sec:experiments}).

Also for the interactive setting, \cite{imt_ijcai2019} have proposed the Blackbox algorithm (\bbox) as a heuristic to apply the \omn~algorithm when the learner's dynamics are unknown---this makes the \bbox~algorithm more widely applicable than \omn. However, this heuristic algorithm is still based on the gradient functional form of the linear \irlmodel~learner model (see Footnote~\ref{foot:bbox}), and does not provide any convergence guarantees. In contrast, our algorithm is derived independent of any specific learner model and we provide a theoretical analysis of our algorithm for different learner models (see Theorems~\ref{prop:curr},~\ref{thm:teach-complexity},~and~\ref{prop:curr-il}). Another crucial difference is that the \bbox~algorithm requires access to the true reward function of the environment, which precludes it from being applied to learner-centric settings where no teacher agent is present. In comparison, our algorithm is applicable to learner-centric settings (see experimental results in Section~\ref{sec:experiment_without_teacher}).

%% file: 1.1_additional_related_work.tex
\subsection{Additional Related Work on Curriculum Design and Teaching}
\label{appendix.relatedwork}

\paragraph{Curriculum design.}
Curriculum design for supervised learning settings has been extensively studied in the literature. Early works present the idea of designing a curriculum comprising of tasks with increasing difficulty to train a machine learning model~\cite{elman1993learning,bengio2009curriculum,zaremba2014learning}. However, these approaches require task-specific knowledge for designing heuristic difficulty measures. Recent works have tackled the problem of automating curriculum design~\cite{graves2017automated,jiang2018mentornet}. There is also an increasing interest in theoretically analyzing the impact of a curriculum (ordering) of training tasks on the convergence of supervised learner models~\cite{weinshall2018curriculum,zhou2018minimax,zhou2021curriculum}. In particular, our work builds on the idea of difficulty scores of the training examples studied in~\cite{weinshall2018curriculum}.

The existing results on curriculum design for sequential decision-making settings are mostly empirical in nature. Similar to the supervised learning settings, the focus on curriculum design for reinforcement learning settings has been shifted from hand-crafted approaches~\cite{asada1996purposive,wu2016training} to automatic methods~\cite{svetlik2017automatic,florensa2017reverse,florensa2018automatic,narvekar2019learning}. We refer the reader to a recent survey~\cite{narvekar2020curriculum} on curriculum design for reinforcement learning. The curriculum learning paradigm has also been studied in psychology literature~\cite{ho2016showing,ho2017intervention,ho2018effective,ho2019,ho2021communication}. One key aspect in these works has been to design algorithms that account for the pedagogical intentions of a teacher, which often aims to explicitly demonstrate specific skills rather than just provide an optimal demonstration for a task. We see our work as complementary to these.

\paragraph{Machine teaching.}
The algorithmic teaching problem considers the interaction between a teacher and a learner where the teacher's objective is to find an optimal training sequence to steer the learner towards a desired goal \cite{goldman1995complexity,liu2017iterative,DBLP:journals/corr/ZhuSingla18,yang2018cooperative}. Most of the work in machine teaching for supervised learning settings is on batch teaching where the teacher provides a batch of teaching examples at once without any adaptation. The question of how a teacher should adaptively select teaching examples for a learner has been addressed recently in supervised learning settings~\cite{melo2018interactive,pmlr-v80-liu18b,chen2018understanding,yeo2019classroom,Hunziker2019forgetful,Mansouri2019preference}.

Furthermore, \cite{walsh2012dynamic,cakmak2012algorithmic,brown2019machine,imt_ijcai2019,Huag2018features,Tschiatschek2019learneraware} have studied algorithmic teaching for sequential decision-making tasks. In particular, \cite{cakmak2012algorithmic,brown2019machine} have proposed batch teaching algorithms for an IRL agent, where the teacher decides the entire set of demonstrations to provide to the learner before any interaction. These teaching algorithms do not leverage any feedback from the learner. In contrast, as discussed in Section \ref{subsec:detailed-comparison}, \cite{imt_ijcai2019} have proposed interactive teaching algorithms (namely \omn~and \bbox) for an IRL agent, where the teacher takes into account how the learner progresses. The works of \cite{Huag2018features,Tschiatschek2019learneraware} are complementary to ours and study algorithmic teaching when the learner has a differet worldview than the teacher or has its own specific preferences.

%% file: 2_problem_setup.tex
\section{Formal Problem Setup}
\label{sec:problem_setup}

Here, we formalize our problem setting which is based on prior work on sequential teaching \cite{liu2017iterative,imt_ijcai2019}.

\paragraph{Environment.}
We consider an environment defined as a Markov Decision Process (MDP) $\mathcal{M} := \brr{\mathcal{S},\mathcal{A},\mathcal{T},\gamma,P_0,R^E}$, where the state and action spaces are denoted by $\mathcal{S}$ and $\mathcal{A}$, respectively. $\mathcal{T}: \mathcal{S} \times \mathcal{S} \times \mathcal{A} \rightarrow \bss{0,1}$ is the transition dynamics, $\gamma$ is the discounting factor, and $P_0: \mathcal{S} \rightarrow \bss{0,1}$ is an initial distribution over states $\mathcal{S}$. A policy $\pi: \mathcal{S} \times \mathcal{A} \rightarrow \bss{0,1}$ is a mapping from a state to a probability distribution over actions. The underlying reward function is given by $R^E: \mathcal{S} \times \mathcal{A} \rightarrow \mathbb{R}$. 

\paragraph{Teacher-learner interaction.} 
We consider a setting with two agents: a teacher and a sequential learner. The teacher has access to the full MDP $\mathcal{M}$ and has a \emph{target policy} $\pi^E$ (e.g., a near-optimal policy w.r.t. $R^E$). The learner knows the MDP $\mathcal{M}$ but not the reward function $R^E$, i.e., has only access to $\mathcal{M} \setminus R^E$. The teacher's goal is to provide an informative sequence of demonstrations to teach the policy $\pi^E$ to the learner. Here, a teacher's demonstration $\xi \small= \bcc{\brr{s_\tau^\xi,a_\tau^\xi}}_{\tau = 0,1,\dots}$ \normalsize is obtained by first choosing an initial state $s_0^\xi \in \mathcal{S}$ (where $P_0(s_0^\xi) > 0$) and then choosing a trajectory, sequence of state-action pairs, obtained by executing the policy $\pi^E$ in the MDP $\mathcal{M}$. The interaction between the teacher and the learner is formally described in Algorithm~\ref{alg:interaction}. For simplicity, we assume that the teacher directly observes the learner’s policy  $\pi_t^L$ at any time $t$. In practice, the teacher could approximately infer the policy $\pi_t^L$ by probing the learner and using Monte Carlo methods.

\paragraph{Generic learner model.} 
Here, we describe a generic learner update rule for Algorithm~\ref{alg:interaction}. Let $\Theta \subseteq \mathbb{R}^d$ be a parameter space. The learner searches for a policy in the following parameterized policy space: $\Pi_\Theta := \bcc{\pi_\theta : \mathcal{S} \times \mathcal{A} \rightarrow \bss{0,1}, \text{ where } \theta \in \Theta}$. For the policy search, the learner sequentially minimizes a loss function $\ell$ that depends on the policy parameter $\theta$ and the demonstration $\xi = \bcc{\brr{s_\tau^\xi,a_\tau^\xi}}_{\tau}$ provided by the teacher. More concretely, we consider $\ell \brr{\xi, \theta} := - \log \mathbb{P}\brr{\xi | \theta}$,
where
\small
$\mathbb{P}\brr{\xi | \theta} = P_0 (s_0^\xi) \cdot \prod_{\tau} \pi_{\theta}\brr{a_\tau^\xi | s_\tau^\xi} \cdot \mathcal{T} \brr{s_{\tau + 1}^\xi | s_\tau^\xi , a_\tau^\xi}$
\normalsize
is the likelihood (probability) of the demonstration $\xi$ under policy $\pi_\theta$ in the MDP $\mathcal{M}$. At time $t$, upon receiving a demonstration $\xi_t$ provided by the teacher, the learner performs the following online projected gradient descent update:
$\wnext ~\gets~ \mathrm{Proj}_\Theta \bss{\w - \eta_t g_t}$,   
where $\eta_t$ is the learning rate, and $g_t = \bss{\nabla_\theta \ell \brr{\xi_t , \theta}}_{\theta = \theta_t}$. Note that the parameter $\theta_1$ reflects the initial knowledge of the learner. Given the learner's current parameter $\theta_t$ at time $t$, the learner's policy is defined as $\pi^L_{t} := \pi_{\w}$.

\begin{algorithm}[t]
    \caption{Teacher-Learner Interaction}
    \begin{algorithmic}[1]
        \State \textbf{Initialization:} Initial knowledge of learner $\pi^L_1$.
        \For{$t = 1,2,\dots$}
            \State Teacher observes the learner's current policy $\pi^L_t$.
            \State Teacher provides demonstration $\xi_t$ to the learner.
            \State Learner updates its policy to $\pi^L_{t+1}$ using $\xi_t$.
        \EndFor{}
    \end{algorithmic}
    \label{alg:interaction}
\end{algorithm}

\paragraph{Teaching objective.}
For any policy $\pi$, the value (total expected reward) of $\pi$ in the MDP $\mathcal{M}$ is defined as
$
V^\pi ~:=~ \sum_{s,a} \sum_{\tau=0}^{\infty} \gamma^\tau \cdot \Prob{S_\tau = s \mid \pi, \mathcal{M}} \cdot \pi \brr{a \mid s} \cdot R^E \brr{s,a} , 
$ where $\Prob{S_\tau = s \mid \pi, \mathcal{M}}$ denotes the probability of visiting the state $s$ after $\tau$ steps by following the policy $\pi$. Let $\pi^L$ denote the learner's final policy at the end of teaching. The performance of the policy $\pi^L$ (w.r.t. $\pi^E$) in $\mathcal{M}$ can be evaluated via $\abs{V^{\pi^E} - V^{\pi^L}}$~\cite{abbeel2004,ziebart2010modeling}. The teaching objective is to ensure that the learner's final policy $\epsilon$-\emph{approximates} the teacher's policy, i.e., $\abs{V^{\pi^E} - V^{\pi^L}} \leq \epsilon$. The teacher aims to provide an optimized sequence of demonstrations $\bcc{\xi_t}_{t=1,2,\dots}$ to the learner to achieve the teaching objective. The teacher's performance is then measured by the number of demonstrations required to achieve this objective. 
Based on existing work \cite{liu2017iterative,imt_ijcai2019}, we assume that $\exists ~ \theta^* \in \Theta$ such that $\pi^E = \pi_{\theta^*}$ (we refer to $\theta^*$ as the \emph{target teaching parameter}). Similar to \cite{imt_ijcai2019}, we assume that a smoothness condition holds in the policy parameter space: $\abs{V^{\pi_{\theta}} - V^{\pi_{\theta'}}} \leq \mathcal{O} \brr{f\brr{\norm{\theta - \theta'}}}  \forall \theta , \theta' \in \Theta$. Then, the teaching objective in terms of $V^\pi$ convergence can be reduced to the convergence in the parameter space, i.e., we can focus on the quantity $\norm{\theta^* - \theta_t}$.

%% file: 3.1_curriculum_algorithms_new.tex
\section{Curriculum Design using Difficulty Scores}
\label{sec:curr-teacher-algo}

\looseness-1In this section, we introduce our curriculum strategy which is based on the concept of \emph{difficulty scores} and is agnostic to the dynamics of the learner.

\paragraph{Difficulty scores.}
We begin by assigning a difficulty score $\Psi_\theta \brr{\xi}$ for any demonstration $\xi$ w.r.t. a parameterized policy $\pi_\theta$ in the MDP $\mathcal{M}$. Inspired by difficulty scores for supervised learning algorithms~\cite{weinshall2018curriculum}, we consider a difficulty score which is directly proportional to the loss function $\ell$, i.e., $\Psi_\theta \brr{\xi} \propto g \brr{\ell \brr{\xi, \theta}}$, for a monotonically increasing function $g$. Setting $g(\cdot)=\exp(\cdot)$ leads to $\Psi_\theta \brr{\xi} = \frac{1}{\prod_{\tau} \pi_{\theta}(a_\tau^\xi | s_\tau^\xi)}$ for MDPs with deterministic transition dynamics. Based on this insight, we define the following difficulty score which we use throughout our work.

\begin{definition}
The difficulty score of a demonstration $\xi$ w.r.t. the policy $\pi_\theta$ in the MDP $\mathcal{M}$ is given by $\Psi_\theta \brr{\xi} := \frac{1}{\prod_{\tau} \pi_{\theta}(a_\tau^\xi | s_\tau^\xi)}$.
\label{def:diff-scores}
\end{definition}
\looseness-1Intuitively, the difficulty score of a demonstration $\xi$ w.r.t. an agent's policy is inversely proportional to the preference of the agent to follow the demonstration. Demonstrations with a higher likelihood under the agent's policy (higher preference) have a lower difficulty score and vice versa. With the above definition, the difficulty scores for any demonstration $\xi$ w.r.t. the teacher's and learner's policies (at any time $t$) are respectively given by $\Psi^E \brr{\xi} := \Psi_{\theta^*} \brr{\xi}$ and $\Psi^L_t \brr{\xi} := \Psi_{\theta_t} \brr{\xi}$.

\paragraph{General curriculum strategy.}
Our curriculum strategy picks the next demonstration $\xi_t$ to provide to the learner based on a preference ranking induced by the teacher's and learner's difficulty scores. The difficulty score of a demonstration $\xi$ w.r.t.~the teacher and learner (at any time t) is denoted by $\Psi^E$ and $\Psi^L_t$ respectively. Specifically, our curriculum strategy is given by:
\begin{equation}
    \xi_t ~\gets~ \argmax_{\xi} \frac{\Psi^L_t \brr{\xi}}{\Psi^E \brr{\xi}}.
    \label{eq:curr_scheme}
\end{equation}

\paragraph{Teacher-centric and learner-centric settings.}
In the teacher-centric setting formalized in Section \ref{sec:problem_setup}, our curriculum strategy utilizes the difficulty scores induced by the learner's current policy $\pi^L_t$ and the teacher's policy $\pi^E$. From Eq.~(\ref{eq:curr_scheme}) and Definition \ref{def:diff-scores}, we obtain the following teacher-centric curriculum strategy: $\xi_t \gets \argmax_{\xi} \prod_{\tau} \frac{\pi^{E}(a^\xi_\tau | s^\xi_\tau)}{\pi^L_t(a^\xi_\tau | s^\xi_\tau)}$.

Additionally, we also consider the learner-centric setting where a teacher agent is not present and the target policy $\pi^E$ is unknown. Here, the learner can benefit from designing a self-curriculum (i.e., automatically ordering demonstrations) based on its current policy $\pi_t^L$. We adapt our curriculum strategy to this setting by utilizing task-specific domain knowledge to define the teacher's difficulty score $\Psi^E (\xi)$ for any demonstration $\xi$. From Eq.~\eqref{eq:curr_scheme}, given the learner's current policy $\pi^L_t$ and the teacher's difficulty score $\Psi^E (\xi)$, the learner-centric curriculum strategy is given as follows: $\xi_t \gets \argmax_{\xi} \frac{1}{\Psi^E (\xi) \prod_{\tau} \pi^L_t(a^\xi_\tau | s^\xi_\tau)}$. 

Note that our curriculum strategy only requires access to the learner's and teacher's policies ($\pi_t^L$ and $\pi^E$) and does not depend on the learner's internal dynamics (i.e, its update rule as mentioned in Section \ref{sec:problem_setup}). This makes our approach more widely applicable to practical applications where it is often possible to infer an agent's policy, but the internal update rule is unknown.

%% file: 3.2_curriculum_learning_theory.tex
\section{Theoretical Analysis of Our Curriculum Strategy}
\label{sec:curr-teacher-theory}

In this section, we present the theoretical analysis of our curriculum strategy for two popular learner models, namely, \irlmodel~and \bcmodel. For our analysis, we consider the teacher-centric setting as introduced in Section \ref{sec:problem_setup}. Our curriculum strategy obtains a preference ranking over the demonstrations to provide to the learner based on the difficulty scores (see Definition~\ref{def:diff-scores}). To this end, we analyze the relationship between the difficulty scores (w.r.t.~the teacher and the learner) of the provided demonstration and the teaching objective (convergence towards the target teaching parameter $\theta^*$) during each sequential update step of the learner. 

Given two difficulty values $\psi^E, \psi^L \in \mathbb{R}$, we define the feasible set of demonstrations at time $t$ as $\mathcal{D}_t \brr{\psi^E, \psi^L} := \bcc{\xi : \Psi^E \brr{\xi} = \psi^E \text{ and } \Psi^L_t \brr{\xi} = \psi^L}$. This set contains all demonstrations $\xi$ for which the teacher’s difficulty score $\Psi^E(\xi)$ is equal to the value $\psi^E$, and the learner’s difficulty score $\Psi^L(\xi)$ is equal to the value $\psi^L$. Let $\Delta_t\brr{\psi^E, \psi^L}$ denote the expected convergence rate of the teaching objective at time $t$, given difficulty values $\psi^E$ and $\psi^L$:
\begin{align}
\Delta_t \brr{\psi^E, \psi^L} 
:= \mathbb{E}_{\xi_t \mid \psi^E, \psi^L}\![\norm{\wopt-\w}^2-\norm{\wopt-\wnext (\xi_t)}^2],
\label{eq:convergence-quantity}
\end{align}
where the expectation is w.r.t.~the uniform distribution over the set $\mathcal{D}_t \brr{\psi^E, \psi^L}$. Below, we analyse the differential effect of $\psi^E$ and $\psi^L$ on $\Delta_t \brr{\psi^E, \psi^L}$, i.e., the effect of picking demonstrations with higher or lower difficulty scores on the learning progress.

\subsection{Analysis for \irlmodel~Learner}
\label{subsec:max-ent}

Here, we consider the popular \irlmodel~learner model~\cite{ziebart2008,ziebart2010modeling,ziebart2013principle} in an MDP $\mathcal{M}$ with deterministic transition dynamics, i.e., $\mathcal{T}: \mathcal{S} \times \mathcal{S} \times \mathcal{A} \rightarrow \bcc{0,1}$. The \irlmodel~learner model uses a parametric reward function $R_\theta: \mathcal{S} \times \mathcal{A} \rightarrow \mathbb{R}$ where $\theta \in \mathbb{R}^d$ is a parameter. The reward function $R_\theta$ also depends on a feature mapping $\phi: \mathcal{S} \times \mathcal{A} \rightarrow \mathbb{R}^{d'}$ which encodes each state-action pair $\brr{s,a}$ by a feature vector $\phi\brr{s,a} \in \mathbb{R}^{d'}$. For our theoretical analysis, we consider $R_\theta$ with a linear form, i.e., $R_\theta\brr{s,a} := \ipp{\theta}{\phi\brr{s,a}}$ and $d=d'$. In our experiments, we go beyond these simplifications and consider environments with stochastic transition dynamics and non-linear reward functions.

Under the \irlmodel~learner model, the parametric policy takes the following soft-Bellman form: $\pi_{\theta}(a|s) = \exp{(Q_{\theta}(s,a) - V_{\theta}(s) )}$, where
\small
$V_{\theta}(s) = \log \sum_{a} \exp{Q_{\theta}(s,a)}$ and 
$Q_{\theta}(s, a) = R_{\theta}(s,a) + \gamma \sum_{s'} \mathcal{T}(s'|s, a) \cdot V_{\theta}(s')$.
\normalsize
For any given $\theta$, the corresponding policy $\pi_\theta$ can be efficiently computed via the Soft-Value-Iteration procedure with reward $R_\theta$ (see  \cite[Algorithm.~9.1]{ziebart2010modeling}). 
For the above setting and a given parameter $\theta$, the probability distribution $\mathbb{P}\brr{\xi | \theta}$ 
over the demonstration $\xi$ takes the closed-form
$\mathbb{P}\brr{\xi | \theta}  = \frac{\exp\brr{\ipp{\theta}{\mu^\xi}}}{Z\brr{\theta}}$,
where $\mu^\xi := \sum_{\tau = 0}^{\infty} \gamma^\tau \phi(s_\tau^\xi, a_\tau^\xi)$ and $Z\brr{\theta}$ is a normalization factor. Then, at time $t$, the gradient  
of the \irlmodel~learner is given by
$g_t = \mu^{\pi_{\theta_t}} - \mu^{\xi_t}$,
where
\small
$\mu^\pi := \sum_{s,a} \sum_{\tau=0}^{\infty} \gamma^\tau \cdot \Prob{S_\tau = s \mid \pi, \mathcal{M}} \cdot \pi \brr{a \mid s} \cdot \phi \brr{s,a}$
\normalsize
is the feature expectation vector of policy $\pi$. We note that our curriculum strategy in Eq.~(\ref{eq:curr_scheme}) is not using knowledge of $g_t$.

For the \irlmodel~learner, we obtain the following theorem, which shows the differential effect of the difficulty scores (w.r.t. the teacher and the learner) on the expected rate of convergence of the teaching objective $\Delta_t \brr{\psi^E, \psi^L}$. We note that \cite{weinshall2018curriculum} obtained similar results for linear regression learner models in the supervised learning setting.
 
\begin{theorem}
\label{prop:curr}
Assume that $\eta_t$ is sufficiently small for all $t$ s.t. $\eta_t \norm{g_t}^2 \ll 2 \abs{\ipp{\wopt - \w}{g_t}}$,
where $g_t$ is the gradient of the \irlmodel~learner. Then, for the \irlmodel~learner, the expected convergence rate of the teaching objective $\Delta_t \brr{\psi^E, \psi^L}$ is:
\begin{itemize}[parsep=0.5pt, leftmargin=*,labelindent=0.5pt]
\vspace{-2mm}
\item monotonically decreasing with value $\psi^E$, i.e., $\frac{\partial \Delta_t}{\partial\psi^E} < 0$, and 
\item monotonically increasing with value $\psi^L$, i.e., $\frac{\partial \Delta_t}{\partial\psi^L} > 0$.
\end{itemize}
\end{theorem}

Theorem~\ref{prop:curr} suggests that choosing demonstrations with lower difficulty score w.r.t.~the teacher’s policy and higher difficulty score w.r.t.~the learner’s policy would lead to faster convergence. Our curriculum strategy in Eq.~\eqref{eq:curr_scheme} induces a preference ranking over demonstrations that aligns with these insights of Theorem~\ref{prop:curr}. Furthermore, the following theorem states that the particular form of combining the two difficulty scores used in curriculum strategy, Eq.~\eqref{eq:curr_scheme}, achieves linear convergence to the teaching objective. This is similar to the state-of-the-art \omn~algorithm based on the IMT framework for sequential learners~\cite{liu2017iterative,imt_ijcai2019}. Importantly, unlike the \omn~algorithm, our curriculum strategy does not rely on specifics of the learner model when selecting demonstrations. 

\begin{theorem}
\label{thm:teach-complexity}
Consider Algorithm~\ref{alg:interaction} with the \irlmodel~learner and our curriculum strategy in Eq.~\eqref{eq:curr_scheme}. Then, the teaching objective $\norm{\wopt - \theta_{t}} \leq \epsilon$ is achieved in $t = \mathcal{O}(\log \frac{1}{\epsilon})$ iterations.
\end{theorem}


In the above theorem, the constant terms suppressed by the $\mathcal{O}(\cdot)$ notation depend on the learning rate of the learner ($\eta_t$), the distance between the learner's initial parameter/knowledge and the target teaching parameter ($\norm{\wopt - \theta_1}$), and the \emph{richness} of the set of demonstrations obtained by executing the policy $\pi^E$ in the MDP $\mathcal{M}$. The \emph{richness} notion is formally discussed in the Appendix.

\subsection{Analysis for \bcmodel~Learner}
\label{subsec:cross-ent}

Next, we consider the \bcmodel~learner model~\cite{bain1995framework,osa2018algorithmic}. In this case, the learner's parametric policy takes the following softmax form:
$\pi_{\theta}(a|s) = \frac{\exp \brr{H_{\theta}(s,a)}}{\sum_{a'} \exp \brr{H_{\theta}(s,a')}}$,
where $H_\theta: \mathcal{S} \times \mathcal{A} \rightarrow \mathbb{R}$ is a parametric scoring function that depends on the parameter $\theta \in \mathbb{R}^d$ and a feature mapping $\phi: \mathcal{S} \times \mathcal{A} \rightarrow \mathbb{R}^{d'}$. For our theoretical analysis, we consider a linear scoring function $H_{\theta}$ of the form $H_\theta\brr{s,a} := \ipp{\theta}{\phi\brr{s,a}}$ (with $d=d'$). Then, at time step $t$, the gradient $g_t$ of the \bcmodel~learner is given by:
$g_t =\small \sum_{\tau=0}^{\infty} \brr{\Expectover{a \sim \pi_{\theta_t} \brr{\cdot | s_\tau^{\xi_t}}}{\phi(s_\tau^{\xi_t}, a)} - \phi(s_\tau^{\xi_t}, a_\tau^{\xi_t})}$. \normalsize
In the experiments, we also consider non-linear scoring functions parameterized by neural networks.

Similar to Theorem~\ref{prop:curr}, we obtain the following theorem for the \bcmodel~learner, which also justifies our curriculum strategy in Eq.~\eqref{eq:curr_scheme}.
\begin{theorem}
\label{prop:curr-il}
Assume that $\eta_t$ is sufficiently small for all $t$ s.t.~$\eta_t \norm{g_t}^2 \ll 2 \abs{\ipp{\wopt - \w}{g_t}}$,
where $g_t$ is the gradient of the \bcmodel~learner. Then, for the \bcmodel~learner, the expected convergence rate of the teaching objective $\Delta_t \brr{\psi^E, \psi^L}$, after first-order approximation, is:
\begin{itemize}[parsep=0.5pt, leftmargin=*,labelindent=0.5pt]
\vspace{-2mm}
\item monotonically decreasing with $\psi^E$, i.e., $\frac{\partial \Delta_t}{\partial\psi^E} < 0$, and 
\item monotonically increasing with $\psi^L$, i.e., $\frac{\partial \Delta_t}{\partial\psi^L} > 0$.
\end{itemize}
\end{theorem}
We note that the proof of Theorem~\ref{prop:curr-il} relies on the first-order Taylor approximation of the term $\sum_{\tau} \log \sum_{a'} \exp \brr{H_\theta \brr{s_\tau^{\xi_t} , a'}}$ around $\theta_t$ (detailed in the Appendix). Due to this approximation, it is more challenging to obtain a convergence result analogous to Theorem~\ref{thm:teach-complexity}.

%% file: 4_driving_experiments_new.tex

\section{Experimental Evaluation: Teacher-Centric Setting}\label{sec:experiments}

\looseness-1Inspired by the works of \cite{ng2000algorithms,levine2010feature,imt_ijcai2019}, we evaluate the performance of our curriculum strategy, Eq.~(\ref{eq:curr_scheme}), in a synthetic car driving environment on \irlmodel~and \bcmodel~learners. In particular, we consider the environment of \cite{imt_ijcai2019} and the teacher-centric setting of Section \ref{sec:problem_setup}. 

\paragraph{Car driving setup.}
\looseness-1Fig.~\ref{fig:templates} illustrates a synthetic car driving environment consisting of $8$ different types of tasks, denoted as \texttt{T0}, \texttt{T1}, $\ldots$, \texttt{T7}. Each type is associated with different driving skills. For instance, \texttt{T0} corresponds to a basic setup representing a traffic-free highway. \texttt{T1} represents a crowded highway. \texttt{T2} has stones on the right lane, whereas \texttt{T3} has a mix of both cars and stones.  Similarly, \texttt{T4} has grass on the right lane, and \texttt{T5} has a mix of both grass and cars. \texttt{T6} and \texttt{T7} introduce more complex features such as pedestrians, police, and HOV (high occupancy vehicles).
The agent starts navigating from an initial state at the bottom of the left lane of each task, and the goal is to reach the top of a lane while avoiding cars, stones, and other obstacles. 
The agent's action space is given by $\mathcal{A} = $ \{\texttt{left}, \texttt{ straight}, \texttt{ right}\}. Action \texttt{left} steers the agent to the left of the current lane. If the agent is already in the leftmost lane when taking action \texttt{left}, then the lane is randomly chosen with uniform probability. We define similar stochastic dynamics for taking action \texttt{right};  action \texttt{straight} means no change in the lane. Irrespective of the action taken, the agent always moves forward.


\begin{figure*}[t!]
\centering
\captionsetup[subfigure]{aboveskip=0pt,belowskip=1pt}
\begin{minipage}{0.60\textwidth}
\centering
	\includegraphics[width=1\linewidth]{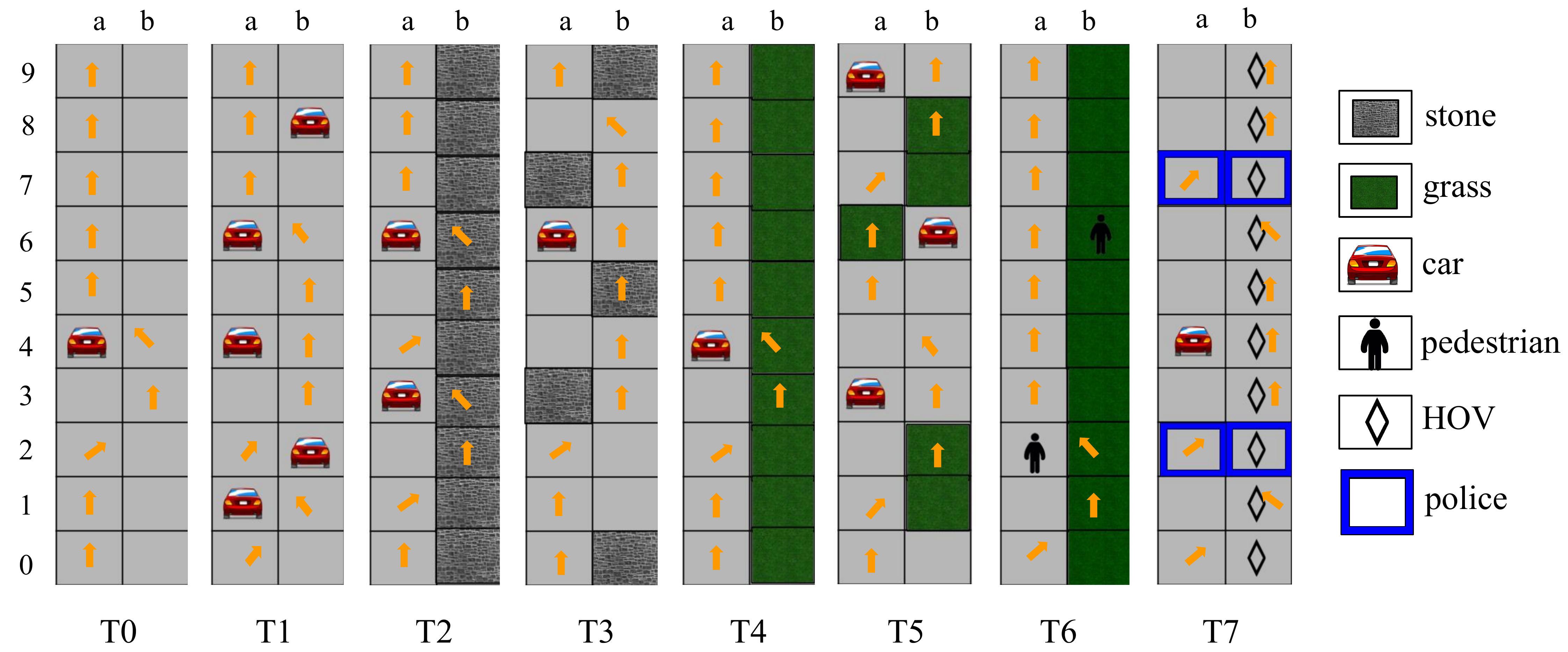}
	\caption{Car environment with $8$ different types of tasks. Arrows represent the path taken by the teacher's policy.
	}
	\label{fig:templates}
\end{minipage}
\qquad
\begin{minipage}{0.3\textwidth}
\centering
\begingroup
\renewcommand{\arraystretch}{0.96} 
	\centering
	\begin{tabular}{|c|c|}
		\hline
		\\[-1.0em]
	    \textbf{$\phi^E(s)$} & \textbf{$w$} \\
		\hline
		\texttt{stone} & -1\\
		\texttt{grass} & -0.5\\
		\texttt{car} & -5\\
		\texttt{ped} & -10\\
		\texttt{car-front} & -2\\
		\texttt{ped-front} & -5\\
		\texttt{HOV} & +1\\
		\texttt{police} & 0\\
		\hline
	\end{tabular}
	\caption{
	Environment features $\phi^E\brr{s}$ and reward weights $w$.
	}
	\label{tab:wstar}
\endgroup
\end{minipage}
\vspace{-2.5mm}
\end{figure*}
\vspace{-1mm}
\paragraph{Environment MDP.}
Based on the above setup, we define the environment MDP, $\mathcal{M}_{\mathrm{car}}$, consisting of $8$ types of tasks, namely \texttt{T0}--\texttt{T7}, and $5$ tasks of each type. Every location in the environment is associated with a state. Each task is of length $10$ and width $2$, leading to a state space of size $5 \times 8 \times 20$.
We consider an action-independent reward function $R^{E_{\mathrm{car}}}$ that is dependent on an underlying feature vector $\phi^E$ (see Fig.~\ref{tab:wstar}). The feature vector of a state $s$, denoted by $\phi^E(s)$, is a binary vector encoding the presence or absence of an object at the state. In this work we have two types of features:
features indicating the type of the current cell as \texttt{stone}, \texttt{grass},  \texttt{car}, \texttt{ped}, \texttt{police}, and \texttt{HOV}, as well as features providing some look-ahead information such as whether there is a car or pedestrian in the immediate front cell (denoted as \texttt{car-front} and \texttt{ped-front}).
Now we explain the reward function $R^{E_{\mathrm{car}}}$. For states in tasks of type \texttt{T0}-\texttt{T6}, the reward 
is given by $\ipp{w}{\phi^E(s)}$ (see Fig.~\ref{tab:wstar}).
Essentially there are different penalties (i.e., negative rewards) for colliding with specific obstacles such as \texttt{stone} and \texttt{car}.
For states in tasks of type \texttt{T7}, there is a reward of value $+1$ for driving on \texttt{HOV}; however, if \texttt{police} is present while driving on \texttt{HOV}, a reward value of $-5$ is obtained. Overall, this results in the reward function $R^{E_{\mathrm{car}}}$ being nonlinear.

\subsection{Teaching Algorithms}
\label{sec:teaching_algorithms}
Here, we introduce the teaching algorithms considered in our experiments. The teacher's near-optimal policy $\pi^E$ is obtained via policy iteration~\cite{sbrl2018}.
The teacher selects demonstrations to provide to the learner using its teaching algorithm.
We compare the performance of our proposed \cur~teacher, which implements our strategy in Eq.~(\ref{eq:curr_scheme}), with the following baselines:
\begin{itemize}[parsep=0.5pt, leftmargin=*,labelindent=0.5pt]
    \item \curt: A variant of our \cur~teacher that samples demonstrations based on the difficulty score $\Psi^E$ alone, and sets $\Psi^L_t$ to constant.
    \item \curl: A similar variant of our \cur~teacher that samples demonstrations based on the difficulty score $\Psi^L_t$ alone, and sets $\Psi^E$ to constant.
    \item \agn: an agnostic teacher that picks demonstrations based on random ordering~\cite{weinshall2018curriculum,imt_ijcai2019}.
    \item \omn: The omniscient teacher is a state-of-the-art algorithm~\cite{imt_ijcai2019,liu2017iterative}, which is applicable only to \irlmodel~learners. \omn~requires complete knowledge of the parameter $\theta^*$, the learner's current parameter $\theta_t$, and the learner's gradients $\eta_t g_t$. Based on this knowledge, the teacher picks demonstrations to directly steer the learner towards $\theta^*$, i.e., by minimizing $\norm{\wopt - (\w - \eta_{t}g_t)}$.
    \item \looseness-1\bbox: The blackbox teacher \cite{imt_ijcai2019} is designed based on the functional form of gradients for the linear MaxEnt-IRL learner model.\footnote{The \bbox~teacher's objective is derived under the assumptions that the reward function can be linearly parameterized as  $\ipp{w^*}{\phi^E(s)}$ and gradients $g_t$ are based on the linear \irlmodel~learner model. Under these assumptions, the teacher's objective can be equivalently written as $|\ipp{w^*}{g_t}|$.\label{foot:bbox}}
    Specifically, the teacher picks a demonstration $\xi$ which maximizes $|\sum_{s', a'} \{ \rho^{\pi^L_t}(s',a') - \rho^{\xi}(s',a')\} R^E(s',a')|$, where $\rho$ denotes state visitation frequency vectors.
    The \bbox~teacher does not require access to \wopt{} or the learner's current parameter \w; however, it requires access to the true reward function $R^E$.
    \item \scot: \looseness-1The \scot~teacher \cite{brown2019machine} aims to find the smallest set of demonstrations required to teach an optimal reward function to the \irlmodel learner. The teacher uses a set cover algorithm to pre-compute the entire curriculum as a batch, prior to training. In our implementation, after having provided the entire batch, the teacher continues providing demonstrations selected at random.

\end{itemize}


\begin{figure*}[t!]
\centering
\centering
\captionsetup[subfigure]{aboveskip=0pt,belowskip=-4pt}
	\begin{subfigure}[b]{.225\textwidth}
    \centering
	{
		\includegraphics[width=\textwidth]{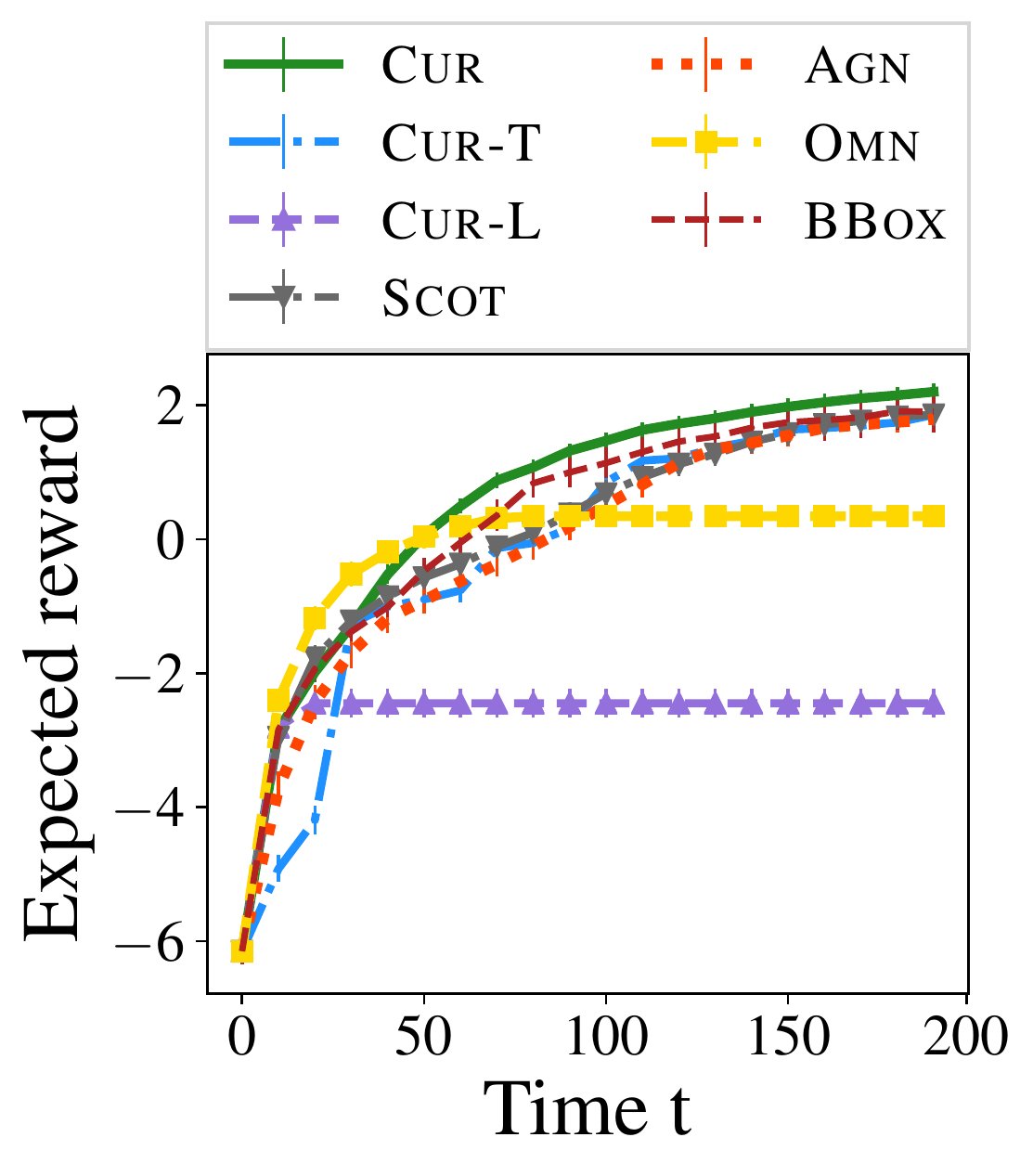}
		\caption{}
		\label{fig:IRL.expected.reward:T0}
	}
	\end{subfigure}
	\ \ \ \
	\begin{subfigure}[b]{.225\textwidth}
		\centering
		{
			\includegraphics[width=\textwidth]{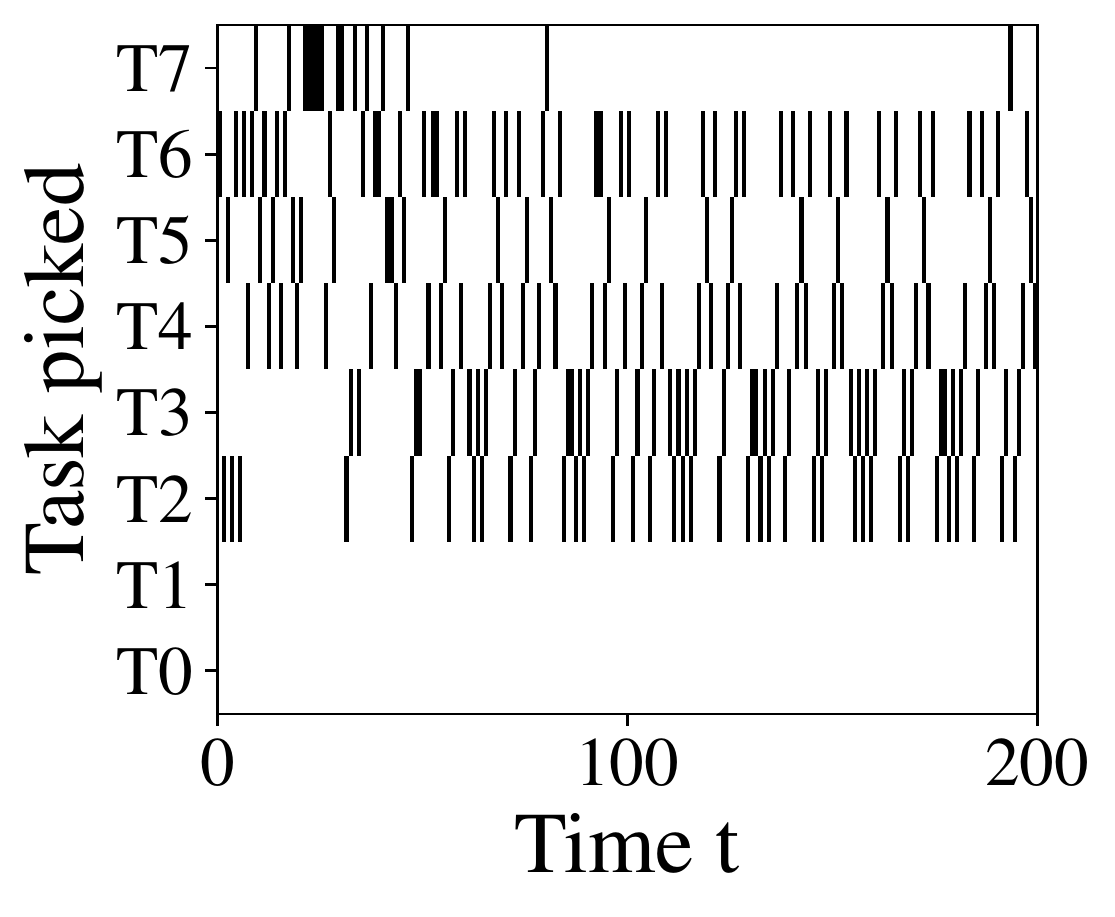}
			\caption{}
			\label{fig:IRL.curriculum:T0}
		}
	\end{subfigure}
	\ \ \ \
	\begin{subfigure}[b]{.22\textwidth}
		\centering
		{
			\includegraphics[width=\textwidth]{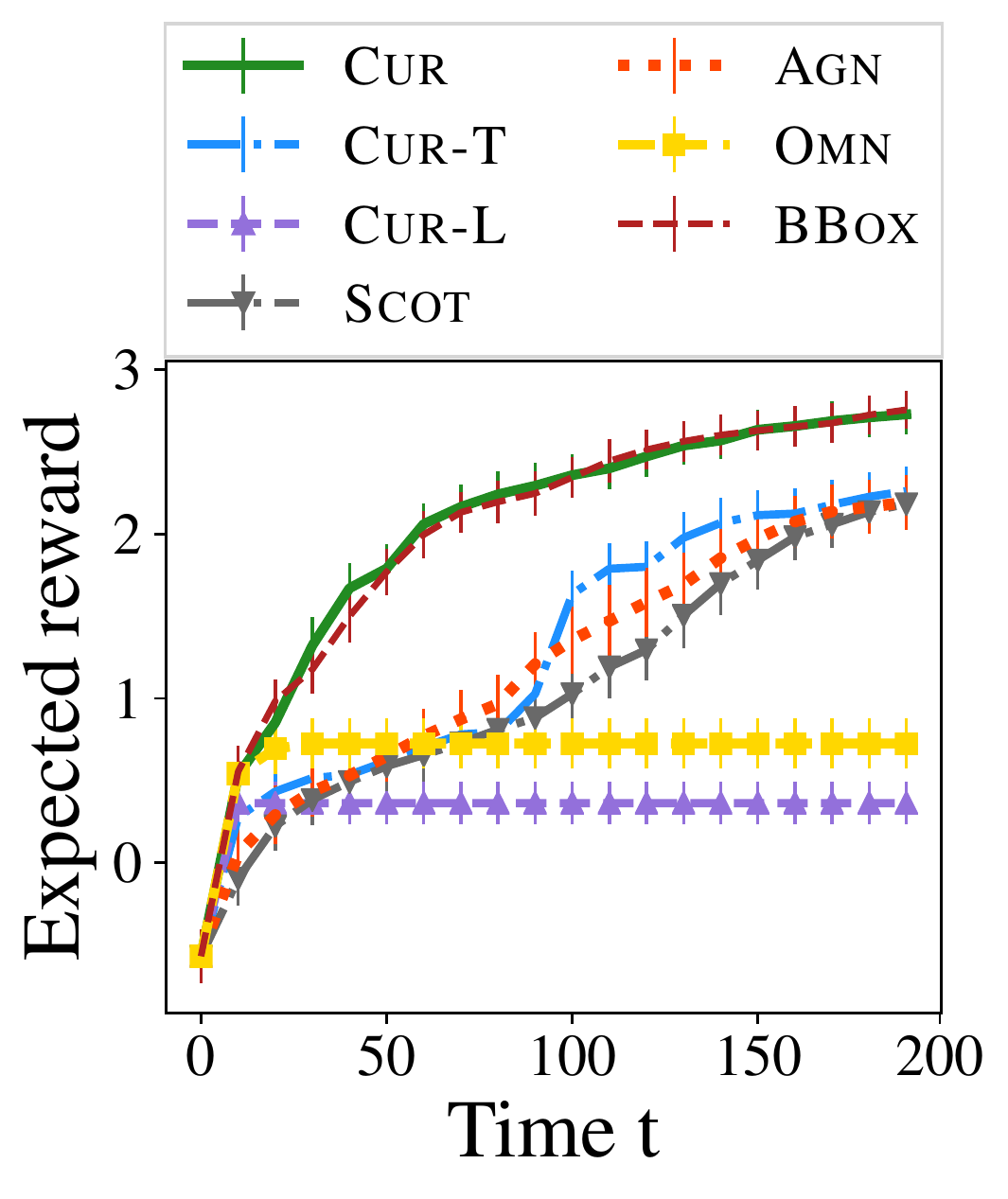}
			\caption{}
			\label{fig:IRL.expected.reward:T0-T3}
		}
	\end{subfigure}
	\ \ \ \
\centering
\begin{subfigure}[b]{0.225\textwidth}
		\centering
		{
			\includegraphics[width=\textwidth]{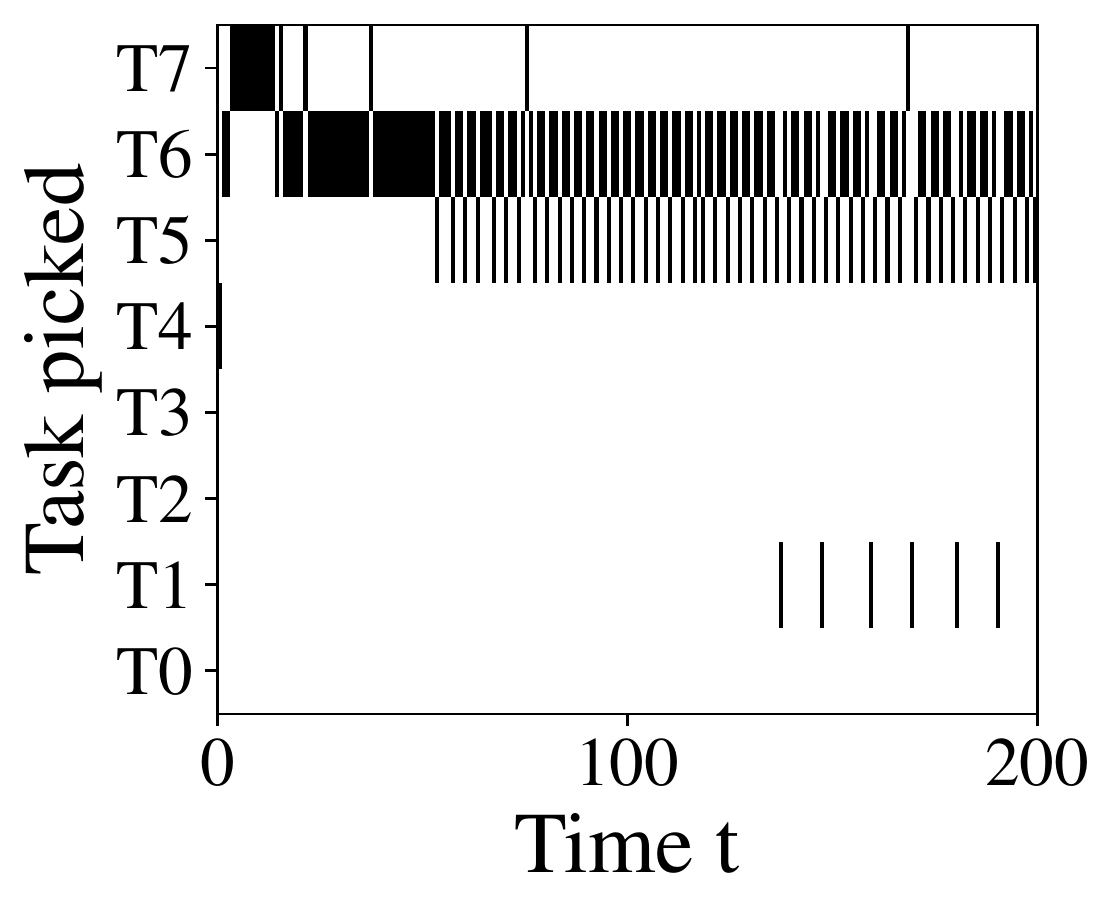}
			\caption{}\label{fig:IRL.curriculum:T0-T3}
		}
	\end{subfigure}
	\caption{Learning curves and curriculum visualization for \irlmodel~learners (with varying initial knowledge) trained on the car driving environment: (a) reward convergence plot and (b) curriculum generated by the \cur~teacher for the learner with initial knowledge of \texttt{T0}; (c) reward convergence plot and (d) curriculum generated by the \cur~teacher for the learner with initial knowledge of \texttt{T0}--\texttt{T3}.}
	\label{fig:IRL.reward_and_curriculum}
    \vspace*{-2.5mm}
\end{figure*}

\begin{figure*}[t!]
\centering
\captionsetup[figure]{aboveskip=0pt,belowskip=-3pt}
\centering
\captionsetup[subfigure]{aboveskip=0pt,belowskip=-4pt}
	\begin{subfigure}[b]{.23\textwidth}
    \centering
	{
		\includegraphics[width=\textwidth]{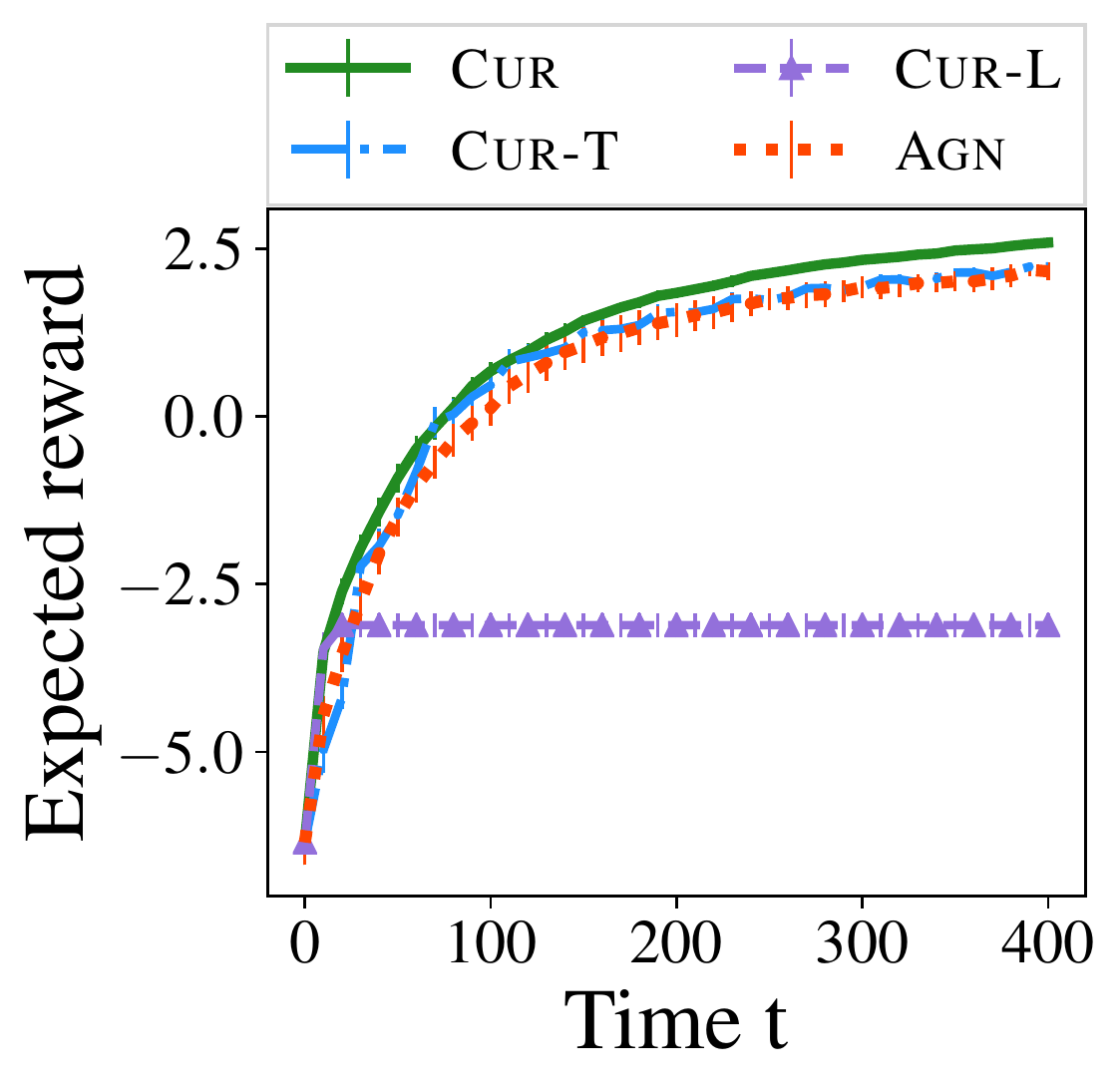}
		\caption{}
		\label{fig:BC.expected.reward:T0}
	}
	\end{subfigure}
	\ \ \ \ 
	\begin{subfigure}[b]{.23\textwidth}
		\centering
		{
			\includegraphics[width=\textwidth]{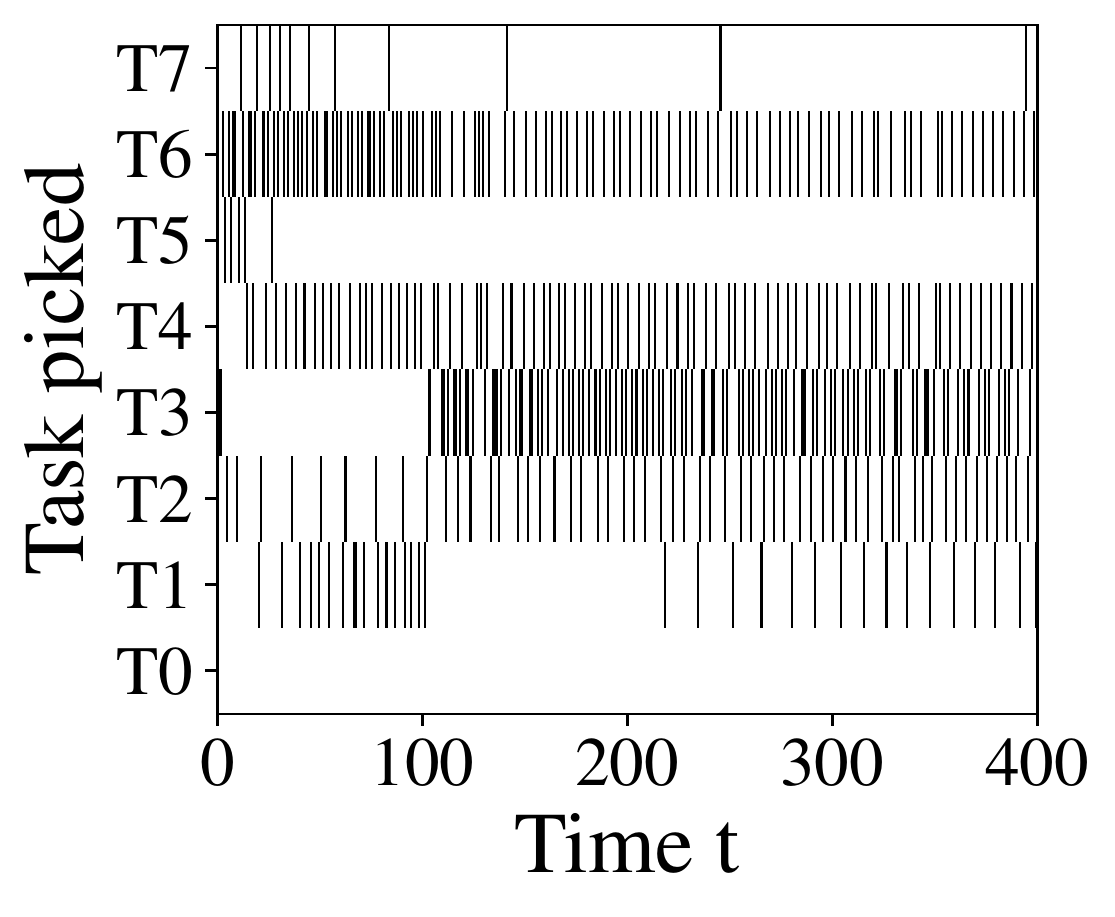}
			\caption{}
			\label{fig:BC.curriculum:T0}
		}
	\end{subfigure}
	\ \ \ \ 
	\begin{subfigure}[b]{.215\textwidth}
		\centering
		{
			\includegraphics[width=\textwidth]{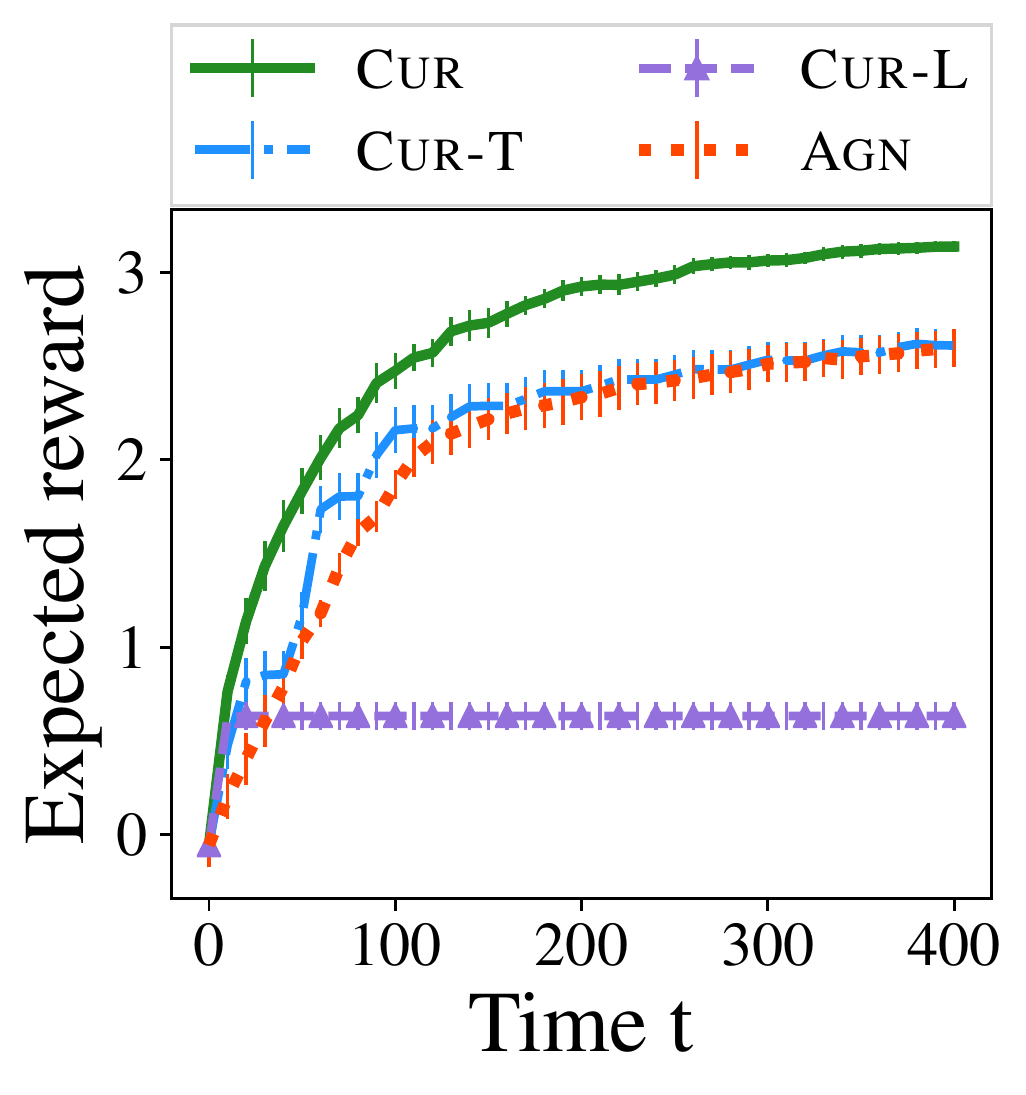}
			\caption{}
			\label{fig:BC.expected.reward:T0-T3}
		}
	\end{subfigure}
	\ \ \ \ 
\centering
\begin{subfigure}[b]{0.23\textwidth}
		\centering
		{
			\includegraphics[width=\textwidth]{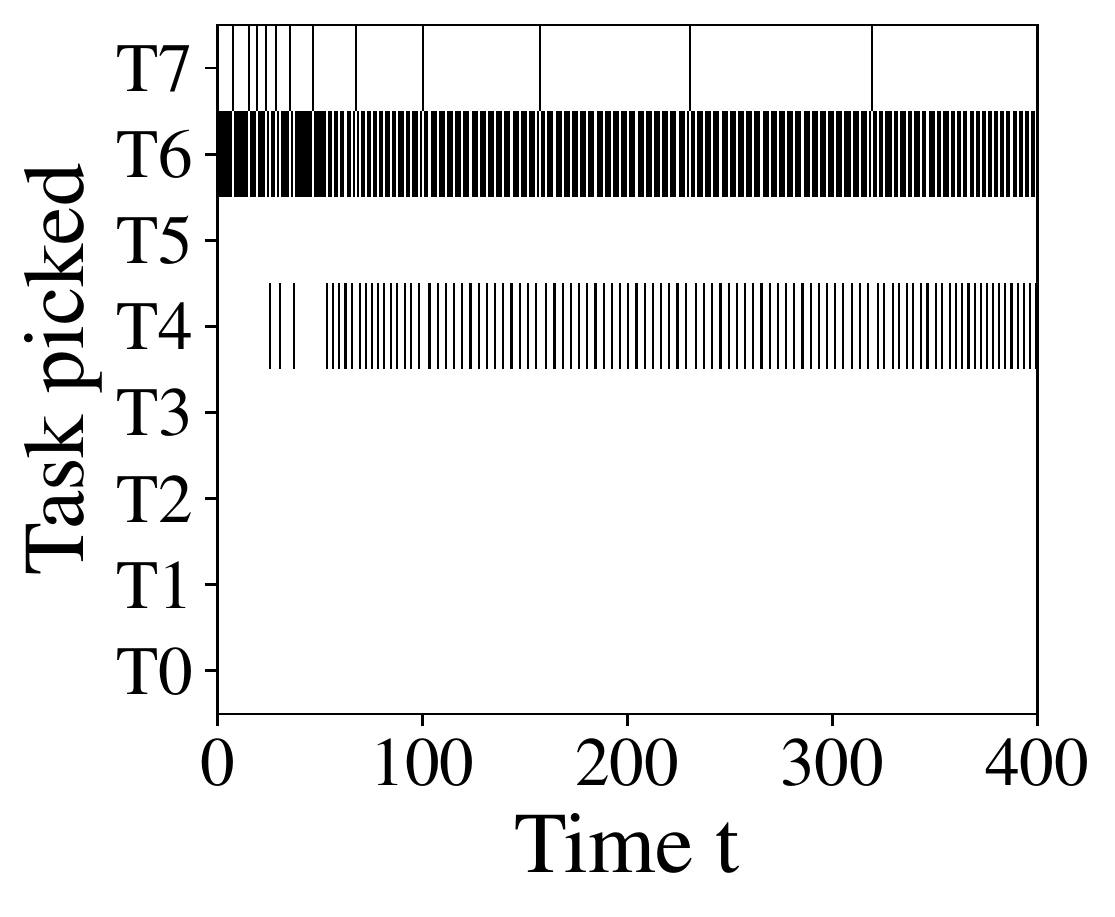}
			\caption{}
			\label{fig:BC.curriculum:T0-T3}
		}
	\end{subfigure}
	\caption{Learning curves and curriculum visualization for \bcmodel~learners (with varying initial knowledge) trained on the car driving environment: (a) reward convergence plot and (b) curriculum generated by the \cur~teacher for the learner with initial knowledge of \texttt{T0}; (c) reward convergence plot and (d) curriculum generated by the \cur~teacher for the learner with initial knowledge of \texttt{T0}--\texttt{T3}.}
	\label{fig:BC.expected.reward}
\end{figure*}

\subsection{Learner Models}
\looseness-1Next, we describe the \irlmodel~and \bcmodel~learner models. For the \irlmodel~learner, we evaluate all the above-mentioned teaching algorithms that include state-of-the-art baselines; for the \bcmodel~learner, we evaluate \cur, \curt, \curl,~and \agn~algorithms.

\vspace{-2mm}
\paragraph{\irlmodel~learner.}
For alignment with the prior state-of-the-art work on teaching sequential \irlmodel~learners \cite{imt_ijcai2019}, we perform \emph{teaching over states} in our experiments. More concretely, at time $t$ the teacher picks a state $s_t$ (where $P_0(s_t) > 0$) and provides all demonstrations starting from $s_t$~to the learner given by \small$\Xi_{s_t} = \bcc{\xi = \bcc{\brr{s^\xi_\tau,a^\xi_\tau}}_{\tau} \text{ s.t. } s^\xi_0 = s_t}$\normalsize.
The gradient $g_t$ of the \irlmodel~learner is then given by
%
$g_t = \mu^{\pi_{\theta_t}, \startstate} - \mu^{\Xi_{\startstate}}$, where (i) $\mu^{\Xi_{\startstate}} := \frac{1}{\abs{\Xi_{\startstate}}} \sum_{\xi \in \Xi_{\startstate}} \mu^\xi$, and (ii) $\mu^{\pi, \startstate}$ is the feature expectation vector of policy $\pi$ with starting state set to $\startstate$ (see Section \ref{subsec:max-ent}).
Based on \cite{imt_ijcai2019}, we consider the learner's feature mapping as $\phi(s, a) = \phi^E(s)$ and the learner uses a non-linear parametric reward function $R^L_\theta\brr{s,a} = \ipp{\theta_{1:d'}}{\phi\brr{s,a}}+\ipp{\theta_{d'+1:2d'}}{\phi\brr{s,a}}^2$ where $d'$ is the dimension of $\phi(s,a)$. As explained in \cite{imt_ijcai2019}, a linear reward representation cannot capture the optimal behaviour for $\mathcal{M}_{\mathrm{car}}$.
We consider learners with varying levels of initial knowledge, i.e., 
the learner is trained on a subset of tasks before the teaching process starts. In this setting, for our curriculum strategy in Eq.~(\ref{eq:curr_scheme}) the difficulty score of a set of demonstrations associated with a starting state $\Xi_{s}$ is computed as the mean difficulty score of individual demonstrations in the set.

\vspace{-3mm}
\paragraph{\bcmodel~learner.}
We consider the \bcmodel~learner model of Section \ref{subsec:cross-ent} as our second learner model. The learner's feature mapping is given by $\phi\brr{s,a} = \E_{s'\sim \mathcal{T}(\cdot | s, a)}[\phi^E(s')]$. A quadratic parametric form is selected for the scoring function, i.e., $H_\theta\brr{s,a} = \ipp{\theta_{1:d'}}{\phi\brr{s,a}}+\ipp{\theta_{d'+1:2d'}}{\phi\brr{s,a}}^2$, where $d'$ is the dimension of $\phi(s, a)$. We consider learners with varying initial knowledge and perform teaching over states similar to the \irlmodel~learner.

\vspace{-0.5mm}
\subsection{Experimental results}
%
\looseness-1Figs. \ref{fig:IRL.expected.reward:T0}, \ref{fig:IRL.expected.reward:T0-T3} and \ref{fig:BC.expected.reward:T0}, \ref{fig:BC.expected.reward:T0-T3} show the convergence of the total expected reward for the \irlmodel~and \bcmodel~learners respectively, averaged over $10$ runs. The \cur~teacher outperforms \omn~despite not requiring information about the learner's dynamics. For non-linear parametric reward functions, the \irlmodel~learner no longer solves a convex optimization problem. As a result, forcing the learner to converge to a fixed parameter doesn't necessarily perform well, as seen by the poor performance of the \omn~teacher in Fig.~\ref{fig:IRL.expected.reward:T0-T3}. The \cur~teacher is competitive with the \bbox~teacher. 
Unlike our \cur~teacher, the \bbox~teacher does require exact access to the true reward function, $R^E$.
The \cur~teacher consistently outperforms the \agn~and \scot~teachers, as well as
both the \curt~and \curl~variants. 

Figs.~\ref{fig:IRL.curriculum:T0}, \ref{fig:IRL.curriculum:T0-T3} and \ref{fig:BC.curriculum:T0}, \ref{fig:BC.curriculum:T0-T3} visualize the curriculum generated by the \cur~teacher for the \irlmodel~and \bcmodel~learners respectively. Here, the curriculum refers to the type of task, \texttt{T0}--\texttt{T7}, associated with the demonstrations provided by the teacher to the learner at time step $t$. For both types of learners we see that at the beginning of training, the teacher focuses on tasks which teach skills the learner is yet to master. For example, in Fig. \ref{fig:BC.curriculum:T0-T3}, the teacher picks tasks \texttt{T4}, \texttt{T6}, and \texttt{T7}, which teaches the learner to avoid grass, pedestrians, and to navigate through police and HOV.
We also notice that the \cur~teacher can identify degradation in performance on previously mastered tasks, e.g., task \texttt{T1} in Fig. \ref{fig:IRL.curriculum:T0-T3}, and corrects for this by picking them again later during training. 

\vspace{-1mm}
\paragraph{Additional results under limited observability.}
\begin{wrapfigure}{r}{0.275\textwidth}
    \vspace{-0.5cm}
    \includegraphics[width=0.275\textwidth]{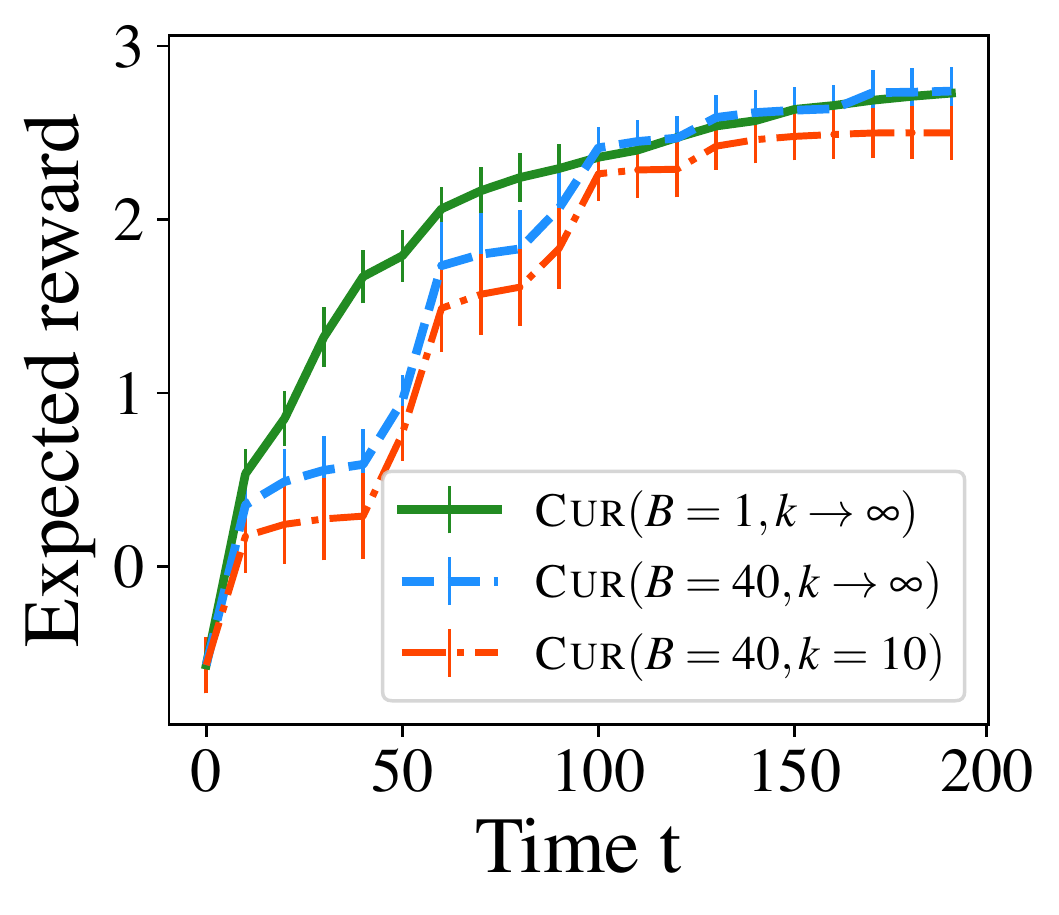}
    \caption{Learning curves for the \irlmodel~learner under limited observability.}
    \label{fig:limited_observability}
\end{wrapfigure}

In the above experiments, we consider the learner's policy to be fully observable by the teacher at every time step. Here, we study the performance of our \cur~teacher under the limited observability setting, similar to \cite{imt_ijcai2019}, where the learner's policy needs to be estimated by probing the learner. The probing process is formally characterized by two parameters, $(B, k)$, where the learner's policy is probed after every $B$ time steps and each probing step corresponds to querying the learner's policy $\pi^L_t$ a total of $k$ times from each state $s\in \mathcal{S}$ in the MDP. The learner's policy, $\pi^L_t(a|s)\;\forall a,s$, is then approximated based on the fraction of the $k$ queries in which the learner performed action $a$ from state $s$. In between every $B$ time steps that the learner is probed, the \cur~teacher does not update its estimate of the learner's policy. We note that the $(B=1, k\to\infty)$ setting corresponds to full observability of the learner.
Fig.~\ref{fig:limited_observability} depicts the performance of the \cur~teacher for different values of $(B,k)$.
Even under limited observability, the \cur~teacher's performance is competitive with the full observability setting. The performance of $(B=40,k\to\infty)$ is even slightly better at certain time steps during later stages of training compared to $(B=1,k\to\infty)$, which is possibly due to the strategy of greedily picking demonstrations not being necessarily optimal. Also, for the limited observability setting it can be interesting to explore approaches that alleviate the need to query the full policy of the learner \cite{jacq2019learning,brown2021value}.

%% file: 5_co_experiments.tex
\section{Experimental Evaluation: Learner-Centric Setting}\label{sec:experiment_without_teacher}

\begin{figure}[t!]
\begin{minipage}{0.48\textwidth}
\centering
\begin{subfigure}[b]{0.43\textwidth}
\centering
\includegraphics[width=\textwidth]{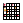}
\caption{}
\label{fig:wsp_grid}
\end{subfigure}
\begin{subfigure}[b]{0.47\textwidth}
\centering
\includegraphics[width=\textwidth]{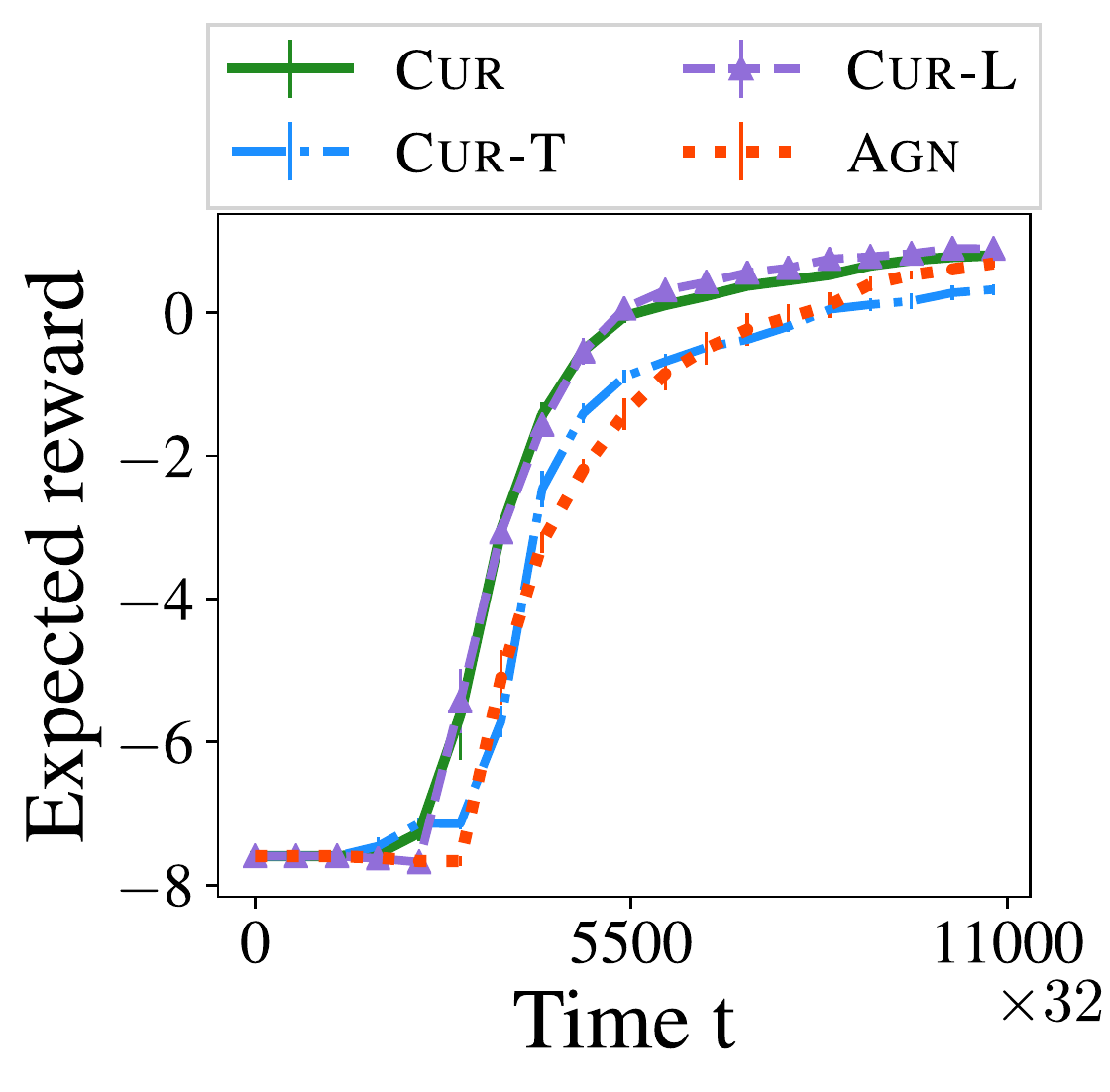}
\caption{}
\label{fig:wsp_convergence}
\end{subfigure}
\caption{Illustration of a shortest path navigation task (left) and  convergence curves (right).}
\label{fig:wsp_figures}
\end{minipage}
\hfill
\begin{minipage}{0.48\textwidth}
\centering
\begin{subfigure}[b]{0.43\textwidth}
\centering
\includegraphics[width=0.98\textwidth]{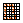}
\caption{}
\label{fig:tsp_grid}
\end{subfigure}
\begin{subfigure}[b]{0.47\textwidth}
\centering
\includegraphics[width=\textwidth]{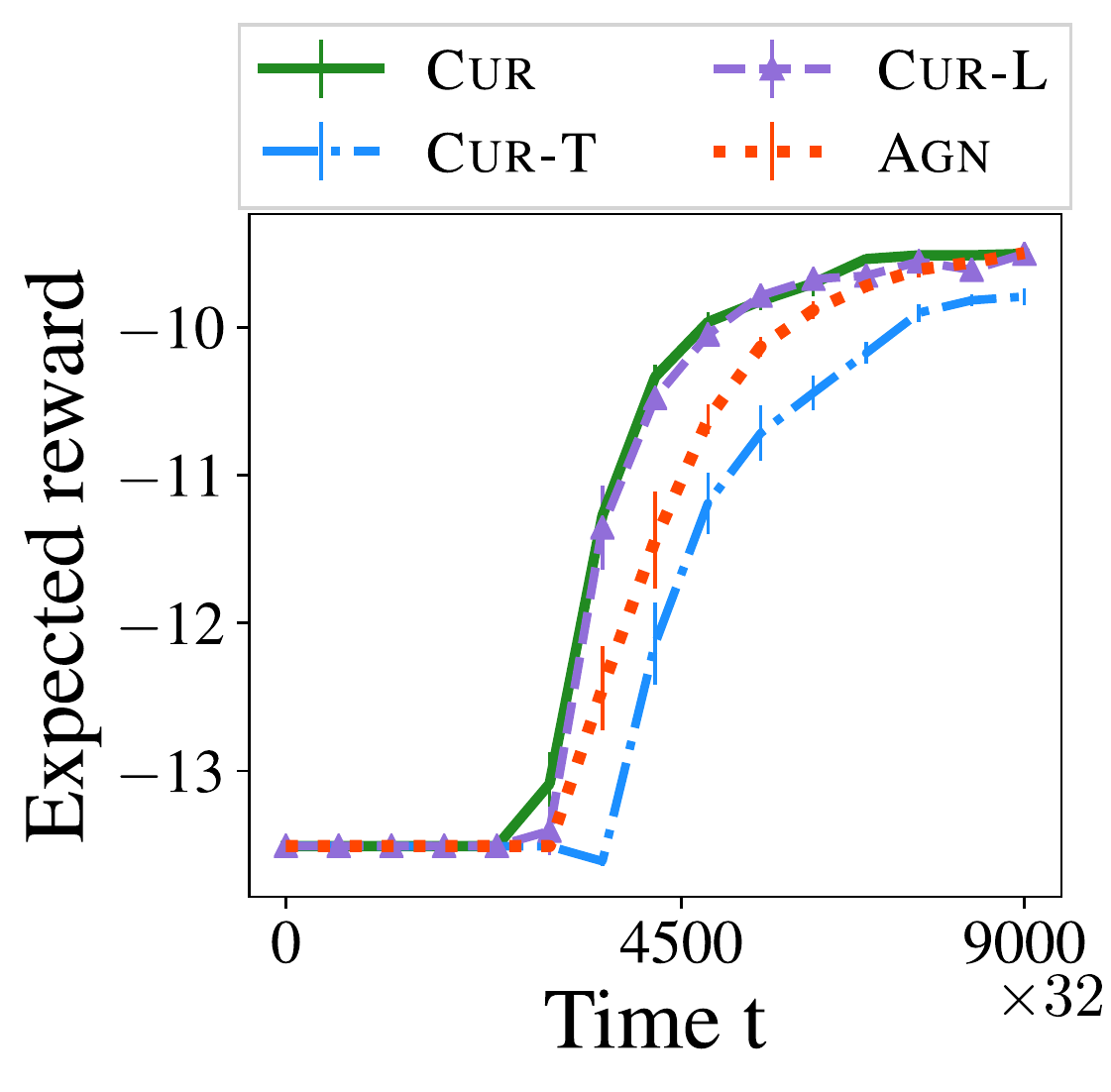}
\caption{}
\label{fig:tsp_convergence}
\end{subfigure}
\caption{Illustration of a travelling salesman navigation task (left) and convergence curves (right).}
\label{fig:tsp_figures}
\end{minipage}
\vspace{-2mm}
\end{figure}

\looseness-1In this section, we evaluate our curriculum strategy in a learner-centric setting, i.e., no teacher agent is present, and the teacher's difficulty $\Psi^E(\xi)$ is expressed by a task-specific difficulty score (see Section \ref{sec:curr-teacher-algo}). We evaluate our approach for training a multi-task neural policy to solve discrete optimization problems. Here, we provide an overview of the results with additional details in the Appendix.

\paragraph{Experiment setup.}
\looseness-1We begin by describing the synthetic navigation-based environments considered in our experiments. Our first navigation environment comprises of tasks based on the shortest path problem \cite{dijkstra1959}. We represent each task with a grid-world (see Fig.~\ref{fig:wsp_grid}) containing goal cells (depicted by stars), and cells with muds and bombs (shown in brown and red respectively). The agent aims to navigate to the closest goal cell while avoiding muds and bombs. Our second navigation environment comprises of tasks inspired by the travelling salesman problem (TSP) \cite{kool2018attention,joshi2019learning} (see Fig.~\ref{fig:tsp_grid}). Again we represent each task with a grid-world, where the agent's goal is to find the shortest tour which visits all goals and returns to its initial location (see Fig.~\ref{fig:tsp_grid}). Orange arrows in Figs.~\ref{fig:wsp_grid} and \ref{fig:tsp_grid} depict the optimal path for the agent. In our experimental setup, we begin by creating a pool of tasks and split them into training and test sets. The curriculum algorithms order the training tasks during the training phase based on their strategy to speed up the learner's progress. The aim of the learner is to learn a multi-task neural policy that can generalize to new unseen tasks in the test set.

\vspace{-1mm}
\paragraph{Curriculum algorithms.}
We compare the performance of four different curriculum algorithms: (i) the \cur~algorithm picks tasks from the training set using Eq.~(\ref{eq:curr_scheme}) where the numerator is $\Psi^L_t$ and the denominator $\Psi^E$ is defined by a task-specific difficulty score (detailed in the Appendix); (ii) the \curl~algorithm picks tasks from the training set using Eq.~(\ref{eq:curr_scheme}) where the numerator is $\Psi^L_t$ and the denominator is set to 1; (iii) the \curt~algorithm picks tasks from the training set using Eq.~(\ref{eq:curr_scheme}) where the numerator is set to 1 and the denominator is set to $\Psi^E$; (iv) the AGN algorithm picks tasks with a uniform distribution over the training set.

\vspace{-1mm}
\paragraph{Learner model.}
\looseness-1We consider a neural \bcmodel~learner (see Section \ref{subsec:cross-ent}). The learner's scoring function $H_{\theta}$ is parameterized by a $6$-layer Convolutional Neural Network (CNN).
The CNN takes as input a feature mapping of the agent's current position in a task, and outputs a score for each action. The learner minimizes the cross-entropy loss between its predictions and the demonstrations.

\vspace{-1mm}
\paragraph{Results.}
Figs.~\ref{fig:wsp_convergence} and \ref{fig:tsp_convergence}, show the reward convergence curves on the test set for the different curriculum algorithms averaged over $5$ runs. The \cur~algorithm leads to faster reward convergence compared to the \agn~algorithm, which is the common approach for training a neural policy. \curl~is competitive with \cur~in this setting which highlights the importance of the learner's difficulty.

%% file: 6_conclusion.tex
\section{Discussion and Conclusions}
\label{sec:conclusion}

We presented a unified curriculum strategy, with theoretical guarantees, for the sequential \irlmodel~and \bcmodel~learner models, based on the concept of difficulty scores. Our proposed strategy is independent of the learner's internal dynamics and is applicable in both teacher-centric and learner-centric settings. Experiments on a synthetic car driving environment and on navigation-based environments demonstrated the effectiveness of our curriculum strategy.

\looseness-1Our work provides theoretical underpinnings of curriculum design for teaching via demonstrations, which can be beneficial in educational applications such as tutoring systems and also for self-curriculum design for imitation learners. As such we do not see any negative societal impact of our work.
Some of the interesting directions for future work include: obtaining convergence bounds for \bcmodel~and other learner models, designing curriculum algorithms for reinforcement learning agents based on the concept of difficulty scores, and designing approaches to efficiently approximate the learner's policy using less queries.


%% file: 9.0_appendix_table-of-contents.tex
\section{List of Appendices}\label{appendix.table-of-contents}
In this section, we provide a brief description of the content in the appendices of the paper.   
\begin{itemize}
\item Appendix~\ref{ap:max-ent} provides proofs for \irlmodel~learner.
\item Appendix~\ref{ap:cross-ent} provides proofs for \bcmodel~learner.
\item Appendix~\ref{ap:experiment_details_nav} provides a detailed description of the synthetic navigation-based experiments.
\end{itemize}

%% file: 9.1_appendix_proofs.tex

\section{Proofs for \irlmodel~Learner}
\label{ap:max-ent}

\subsection{Auxiliary Lemma}

\begin{lemma}
Consider the \irlmodel~learner defined in Section~\ref{subsec:max-ent}. Then, at time step $t$, we have
\[
- \ipp{\wopt - \w}{g_t} ~=~ \log \frac{\Psi^L_t\brr{\xi_t}}{\Psi^E\brr{\xi_t}} + K_t ,
\]
where $K_t = \log \frac{Z\brr{\wopt}}{Z\brr{\w}} - \ipp{\wopt - \w}{\mu^{\pi_{\theta_t}}}$ is a constant independent of $\xi_t$.
\label{helper-lemma}
\end{lemma}

\begin{proof}
Consider the following:
\begin{align*}
\Psi_\theta \brr{\xi_t} ~=~& \frac{1}{\prod_{\tau} \pi_{\theta}\brr{a_\tau^{\xi_t} | s_\tau^{\xi_t}}} \\
~=~& \frac{P_0 (s_0^{\xi_t})}{P_0 (s_0^{\xi_t}) \cdot \prod_{\tau} \pi_{\theta}\brr{a_\tau^{\xi_t} | s_\tau^{\xi_t}}} \\
~\stackrel{(i)}{=}~& \frac{P_0 (s_0^{\xi_t})}{P_0 (s_0^{\xi_t}) \cdot \prod_{\tau} \pi_{\theta}\brr{a_\tau^{\xi_t} | s_\tau^{\xi_t}} \cdot \mathcal{T} \brr{s_{\tau + 1}^{\xi_t} | s_\tau^{\xi_t} , a_\tau^{\xi_t}}} \\
~=~& \frac{P_0 (s_0^{\xi_t})}{\mathbb{P}\brr{\xi_t | \theta}}  \\
~=~& \frac{P_0 (s_0^{\xi_t}) \cdot Z\brr{\theta}}{\exp\brr{\ipp{\theta}{\mu^{\xi_t}}}}
\end{align*}
where (i) is due to the deterministic transition dynamics of the MDP. Thus, we have:
\[
\ipp{\theta}{\mu^{\xi_t}} ~=~ \log \frac{P_0 (s_0^{\xi_t}) \cdot Z\brr{\theta}}{\Psi_\theta \brr{\xi_t}} .
\]
Then, we get:
\begin{align*}
\ipp{\wopt - \w}{\mu^{\xi_t} - \mu^{\pi_{\theta_t}}} ~=~& \log \frac{P_0 (s_0^{\xi_t}) \cdot Z\brr{\wopt}}{\Psi_{\wopt} \brr{\xi_t}} - \log \frac{P_0 (s_0^{\xi_t}) \cdot Z\brr{\w}}{\Psi_{\w} \brr{\xi_t}} - \ipp{\wopt - \w}{\mu^{\pi_{\theta_t}}} \\
~=~& \log \frac{\Psi_{\w} \brr{\xi_t}}{\Psi_{\wopt} \brr{\xi_t}} \cdot \frac{Z\brr{\wopt}}{Z\brr{\w}} - \ipp{\wopt - \w}{\mu^{\pi_{\theta_t}}} \\
~=~& \log \frac{\Psi^L_t\brr{\xi_t}}{\Psi^E\brr{\xi_t}} + K_t ,
\end{align*}
where $K_t = \log \frac{Z\brr{\wopt}}{Z\brr{\w}} - \ipp{\wopt - \w}{\mu^{\pi_{\theta_t}}}$ is a constant independent of $\xi_t$.
\end{proof}

\subsection{Proof of Theorem~\ref{prop:curr}}
\label{ap:teaching_objective_derivation}

\paragraph{Technical conditions.}

Let $\Theta = \mathbb{R}^d$, and for each time step $t$, the learning rate $\eta_t$ satisfies the following condition:
\begin{equation}
\eta_t^2 \norm{g_t}^2 ~\ll~ 2 \eta_t \abs{\ipp{\wopt - \w}{g_t}} , 
\label{eq:eta-condition-irl}
\end{equation}
where $g_t$ is given in Section \ref{subsec:max-ent}. We decompose the gradient as $g_t = \kappa \brr{\wopt - \w} + \delta$, where $\delta \perp \brr{\wopt - \w}$, and $\kappa \in \mathbb{R}$. Then, the above condition can be reduced to the following:
\begin{align*}
\eta_t^2 \brr{\abs{\kappa}^2 \norm{\wopt - \w}^2 + \norm{\delta}^2} ~\ll~& 2 \eta_t \abs{\kappa} \norm{\wopt - \w}^2 \\
\implies \quad \eta_t ~\ll~& \frac{2 \abs{\kappa} \norm{\wopt - \w}^2}{\abs{\kappa}^2 \norm{\wopt - \w}^2 + \norm{\delta}^2} .
\end{align*}
When the gradient $g_t$ primarily aligns with $\pm \brr{\wopt - \w}$, and has a small magnitude to control variance, the above condition further simplifies as follows:
\[
\eta_t ~\ll~ \frac{2}{\abs{\kappa}} .
\]
Smaller values of $\abs{\kappa}$ would impose less stringent condition on $\eta_t$. From Lemma~\ref{helper-lemma}, one can easily observe that our curriculum strategy indeed aims to align the gradient $g_t$ with $- \brr{\wopt - \w}$:  
\[
\argmax_\xi \frac{\Psi^L_t\brr{\xi}}{\Psi^E\brr{\xi}} ~=~ \argmax_\xi \log \frac{\Psi^L_t\brr{\xi}}{\Psi^E\brr{\xi}} ~=~ \argmax_\xi \bcc{- \ipp{\wopt - \w}{g_t\brr{\xi}}} . 
\]
We remark that these technical conditions are only required for our theoretical analysis, and not for our experiments. 

\begin{proof}
Consider the following:
\begin{align}
\Delta_t \brr{\psi^E, \psi^L} ~=~& \Expectover{\xi_t \mid \psi^E, \psi^L}{\norm{\wopt - \w}^2 - \norm{\wopt - \wnext \brr{\xi_t}}^2} \nonumber \\
~=~& \Expectover{\xi_t \mid \psi^E, \psi^L}{\norm{\wopt - \w}^2 - \norm{\wopt - \w + \eta_t g_t}^2} \nonumber \\
~=~& \Expectover{\xi_t \mid \psi^E, \psi^L}{- \eta_t^2 \norm{g_t}^2 - 2 \eta_t \ipp{\wopt - \w}{g_t}} \nonumber \\
~\stackrel{(i)}{\approx}~& 2 \eta_t \Expectover{\xi_t \mid \psi^E, \psi^L}{- \ipp{\wopt - \w}{g_t}} \nonumber \\
~\stackrel{(ii)}{=}~& 2 \eta_t \Expectover{\xi_t \mid \psi^E, \psi^L}{\log \frac{\Psi^L_t\brr{\xi_t}}{\Psi^E\brr{\xi_t}} + K_t} \nonumber \\
~=~& 2 \eta_t \log \frac{\psi^L}{\psi^E} + 2 \eta_t K_t , \label{eq:irl_conv_rate_diff}
\end{align}
where the approximation (i) is due to Eq.~\eqref{eq:eta-condition-irl}, and (ii) is due to Lemma~\ref{helper-lemma}. Then, from~\eqref{eq:irl_conv_rate_diff}, we have:
\begin{align*}
\frac{\partial \Delta_t}{\partial\psi^E} ~\approx~& - \frac{2 \eta_t}{\psi^E} ~<~ 0 , \text{ and } \\
\frac{\partial \Delta_t}{\partial\psi^L} ~\approx~& \frac{2 \eta_t}{\psi^L} ~>~ 0 .
\end{align*}
\end{proof}


\subsection{Proof of Theorem~\ref{thm:teach-complexity}}
\label{ap:teaching_complexity}

\begin{proof}
From Lemma~\ref{helper-lemma}, we have that 
\[
\argmax_\xi \ipp{\wopt - \w}{\mu^{\xi}} ~=~ \argmax_\xi \log \frac{\Psi^L_t\brr{\xi}}{\Psi^E\brr{\xi}} ~=~ \argmax_\xi \frac{\Psi^L_t\brr{\xi}}{\Psi^E\brr{\xi}} .
\]
Thus, our curriculum teaching algorithm picks the demonstration to provide by optimizing the following objective:
\[
\xi_t ~\gets~ \argmax_\xi \ipp{\wopt - \w}{\mu^{\xi}} . 
\]
For a bounded feature mapping $\phi$, we have that $\norm{\mu^\xi} \leq L$, $\forall \xi$. Any optimal solution $\xi_t$ to the above problem satisfies: $\mu^{\xi_t} = \frac{L}{\norm{\wopt - \w}} \brr{\wopt - \w}$. Since in our setting the teacher's demonstrations are restricted to trajectories obtained by executing policy $\pi^E$ in the MDP $\mathcal{M}$, we assume that within the set of available teacher's demonstrations, the optimal feature vector has the following form~\cite{liu2017iterative,imt_ijcai2019}:
\[
\mu^{\xi_t} ~=~ \beta_t \brr{\wopt - \w} + \delta_t,
\]
where $\beta_t \in \bss{0, \frac{L}{\norm{\wopt - \w}}}$ bounds the magnitude of the gradient in the desired direction of $\brr{\wopt - \w}$, and $\delta_t$ represents the deviation from the desired direction, s.t. $\Delta = \max_t \norm{\delta_t}$. We further define the following terms:  $z_{\max} = \max_t \norm{\wopt - \w}$, $\eta_{\max} = \max_t \eta_t$, and $\beta = \min_t \eta_t \beta_t$. 

Consider the following:
\begin{align*}
\norm{\wopt - \wnext}^2 ~=~& \norm{\wopt - \brr{\w + \eta_t \mu^{\xi_t} - \eta_t \mu^{\pi^L_t}}}^2 \\
~=~& \norm{\wopt - \w}^2 + \eta_t^2 \norm{\mu^{\xi_t} - \mu^{\pi^L_t}}^2 - 2 \eta_t \ipp{\wopt - \w}{\mu^{\xi_t} - \mu^{\pi^L_t}} \\
~=~& \norm{\wopt - \w}^2 + \eta_t^2 \norm{\beta_t \brr{\wopt - \w} + \delta_t - \mu^{\pi^L_t}}^2 - 2 \eta_t \ipp{\wopt - \w}{\beta_t \brr{\wopt - \w} + \delta_t - \mu^{\pi^L_t}} \\
~=~& \norm{\wopt - \w}^2 + \eta_t^2 \beta_t^2 \norm{\wopt - \w}^2 + \eta_t^2 \norm{\delta_t}^2 + \eta_t^2 \norm{\mu^{\pi^L_t}}^2 + 2 \eta_t^2 \beta_t \ipp{\wopt - \w}{\delta_t} - 2 \eta_t^2 \beta_t \ipp{\wopt - \w}{\mu^{\pi^L_t}} \\
& - 2 \eta_t^2 \ipp{\delta_t}{\mu^{\pi^L_t}} - 2 \eta_t \beta_t \norm{\wopt - \w}^2 - 2 \eta_t \ipp{\wopt - \w}{\delta_t} + 2 \eta_t \ipp{\wopt - \w}{\mu^{\pi^L_t}} \\
~\stackrel{(i)}{\leq}~& \brr{1 + \eta_t^2 \beta_t^2 - 2 \eta_t \beta_t} \norm{\wopt - \w}^2 + \eta_t^2 \bss{\Delta^2 + L^2} \\
& + 2 \eta_t \brr{1 - \eta_t \beta_t} \Delta \norm{\wopt - \w} + 2 \eta_t \brr{1 - \eta_t \beta_t} L \norm{\wopt - \w} + 2 \eta_t^2 \Delta L \\
~\stackrel{(ii)}{\leq}~& \brr{1 - \eta_t \beta_t}^2 \norm{\wopt - \w}^2 + \eta_t^2 \brr{\Delta + L}^2 + 2 \eta_t \brr{1 - \eta_t \beta_t} \brr{\Delta + L} z_{\max} \\
~\stackrel{(iii)}{\leq}~& \brr{1 - \beta}^2 \norm{\wopt - \w}^2 + \eta_{\max}^2 \brr{\Delta + L}^2 + 2 \eta_{\max} \brr{1 - \beta} \brr{\Delta + L} z_{\max} \\
~\stackrel{(iv)}{\leq}~& \brr{1 - \beta}^2 \norm{\wopt - \w}^2 + \eta_{\max} \bcc{1 + 2 \brr{1 - \beta} z_{\max}} \brr{\Delta + L} ,
\end{align*}
where (i) uses the inequalities $\norm{\mu^{\xi_t}} \leq L$, and $\norm{\delta_t} \leq \Delta$, along with the Cauchy-Schwarz inequality; (ii) utilizes the fact that $\norm{\wopt - \w} \leq z_{\max}$; (iii) is obtained by substituting $\beta = \min_t \eta_t \beta_t$, and $\eta_{\max} = \max_t \eta_t$; (iv) is obtained when $\eta_{\max} \brr{\Delta + L} \leq 1$. Note that the inequality (i) is valid when $1 - \eta_t \beta_t > 0$, $\forall t$.

With the inequality $\sqrt{a + b} \leq \sqrt{a} + \sqrt{b}$ for positive $a, b$, and utilizing recurrence, we obtain:
\begin{align*}
\norm{\wopt - \wnext} ~\leq~& \brr{1 - \beta} \norm{\wopt - \w} + \sqrt{\eta_{\max} \bcc{1 + 2 \brr{1 - \beta} z_{\max}} \brr{\Delta + L}} \\
~\leq~& \brr{1 - \beta}^t \norm{\wopt - \theta_1} + \sqrt{\eta_{\max} \bcc{1 + 2 \brr{1 - \beta} z_{\max}} \brr{\Delta + L}} \sum_{s=0}^{\infty}\brr{1-\beta}^s \\
~=~& \brr{1 - \beta}^t \norm{\wopt - \theta_1} + \sqrt{\eta_{\max} \bcc{1 + 2 \brr{1 - \beta} z_{\max}} \brr{\Delta + L}} \cdot \frac{1}{\beta} \\
~\leq~& \frac{\epsilon}{2} + \frac{\epsilon}{2} ~=~ \epsilon ,
\end{align*}
for $t = \brr{\log \frac{1}{1-\beta}}^{-1} \log \frac{2 \norm{\wopt - \theta_1}}{\epsilon} = \mathcal{O} \brr{\log \frac{1}{\epsilon}}$, and $\eta_{\max}\brr{\Delta + L} \leq \frac{{\epsilon}^2 \beta^2}{4 \bcc{1 + 2 \brr{1 - \beta} z_{\max}}}$.
\end{proof}


\section{Proofs for \bcmodel~Learner}
\label{ap:cross-ent}

\subsection{Auxiliary Lemma}

\begin{lemma}
Consider the \bcmodel~learner defined in Section~\ref{subsec:cross-ent}. Then, at time step $t$, we have
\[
- \ipp{\wopt - \w}{g_t} ~\approx~ \log \frac{\Psi^L_t\brr{\xi_t}}{\Psi^E\brr{\xi_t}} .
\]
\label{helper-lemma-cross}
\end{lemma}

\begin{proof}
Consider the following:
\begin{align*}
\log \Psi_\theta \brr{\xi_t} ~=~& - \log \prod_{\tau} \pi_{\theta}\brr{a_\tau^{\xi_t} | s_\tau^{\xi_t}} \\
~=~& - \sum_{\tau} \log \pi_{\theta}\brr{a_\tau^{\xi_t} | s_\tau^{\xi_t}} \\
~=~& \sum_{\tau} \log \sum_{a'} \exp \brr{H_\theta \brr{s_\tau^{\xi_t} , a'}} - \sum_{\tau} H_\theta \brr{s_\tau^{\xi_t} , a_\tau^{\xi_t}} \\
~\stackrel{(i)}{\approx}~& \sum_{\tau} \log \sum_{a'} \exp \brr{H_{\theta_t} \brr{s_\tau^{\xi_t} , a'}} - \ipp{\theta - \theta_t}{\sum_\tau \Expectover{a' \sim \pi_{\theta_t} \brr{\cdot | s_\tau^{\xi_t}}}{\phi(s_\tau^{\xi_t}, a')}} - \ipp{\theta}{\sum_\tau \phi(s_\tau^{\xi_t}, a_\tau^{\xi_t})}
\end{align*}
where (i) is due to the first-order Taylor approximation of $\sum_{\tau} \log \sum_{a'} \exp \brr{H_\theta \brr{s_\tau^{\xi_t} , a'}}$ around $\theta_t$. Then, we have:
\begin{align*}
\log \frac{\Psi_{\theta_t} \brr{\xi_t}}{\Psi_{\theta^*} \brr{\xi_t}} ~=~& \log \Psi_{\theta_t} \brr{\xi_t} - \log \Psi_{\theta^*} \brr{\xi_t} \\
~\approx~& \ipp{\theta^* - \theta_t}{\sum_\tau \phi(s_\tau^{\xi_t}, a_\tau^{\xi_t}) - \sum_\tau \Expectover{a' \sim \pi_{\theta_t} \brr{\cdot | s_\tau^{\xi_t}}}{\phi(s_\tau^{\xi_t}, a')}} \\
~=~& - \ipp{\wopt - \w}{g_t}
\end{align*}
\end{proof}

\subsection{Proof of Theorem~\ref{prop:curr-il}}
\label{ap:bc-curr-theorem}

\paragraph{Technical conditions.}

Let $\Theta = \mathbb{R}^d$, and for each time step $t$, the learning rate $\eta_t$ satisfies the following condition:
\begin{equation}
\eta_t^2 \norm{g_t}^2 ~\ll~ 2 \eta_t \abs{\ipp{\wopt - \w}{g_t}} , 
\label{eq:eta-condition-bc}
\end{equation}
where $g_t$ is the gradient of the \bcmodel~learner as given in section \ref{subsec:cross-ent}. We can further simplify the above condition, similar to Section~\ref{ap:teaching_objective_derivation}.

\begin{proof}
Consider the following:
\begin{align}
\Delta_t \brr{\psi^E, \psi^L} ~=~& \Expectover{\xi_t \mid \psi^E, \psi^L}{\norm{\wopt - \w}^2 - \norm{\wopt - \wnext \brr{\xi_t}}^2} \nonumber \\
~=~& \Expectover{\xi_t \mid \psi^E, \psi^L}{\norm{\wopt - \w}^2 - \norm{\wopt - \w + \eta_t g_t}^2} \nonumber \\
~=~& \Expectover{\xi_t \mid \psi^E, \psi^L}{- \eta_t^2 \norm{g_t}^2 - 2 \eta_t \ipp{\wopt - \w}{g_t}} \nonumber \\
~\stackrel{(i)}{\approx}~& 2 \eta_t \Expectover{\xi_t \mid \psi^E, \psi^L}{- \ipp{\wopt - \w}{g_t}} \nonumber \\
~\stackrel{(ii)}{\approx}~& 2 \eta_t \Expectover{\xi_t \mid \psi^E, \psi^L}{\log \frac{\Psi^L_t\brr{\xi_t}}{\Psi^E\brr{\xi_t}}} \nonumber \\
~=~& 2 \eta_t \log \frac{\psi^L}{\psi^E} , \label{eq:il_conv_rate_diff}
\end{align}
where the approximation (i) is due to Eq.~\eqref{eq:eta-condition-bc}, and (ii) is due to Lemma~\ref{helper-lemma-cross}. Then, from~\eqref{eq:il_conv_rate_diff}, we have:
\begin{align*}
\frac{\partial \Delta_t}{\partial\psi^E} ~\approx~& - \frac{2 \eta_t}{\psi^E} ~<~ 0 , \text{ and } \\
\frac{\partial \Delta_t}{\partial\psi^L} ~\approx~& \frac{2 \eta_t}{\psi^L} ~>~ 0 .
\end{align*}
\end{proof}




%% file: 9.4_appendix_additional_experiment_details.tex
\section{Additional Details for Learner-Centric Experiments}\label{ap:experiment_details_nav}

\looseness-1In this appendix, we present additional experimental details for the synthetic navigation-based environments considered in Section \ref{sec:experiment_without_teacher}.

\subsection{Environment MDPs}

We first formally define the environment MDPs for the shortest path and TSP inspired environments described in Section \ref{sec:experiment_without_teacher}.

\begin{table}[b]
    \begin{minipage}[b]{0.5\textwidth}
    \centering
  \begin{tabular}{|c|}
    \hline
    Agent facing North \\
    \hline
    Agent facing South \\
    \hline
    Agent facing West \\
    \hline
    Agent facing East \\
    \hline
    Mud \\
    \hline
    Bomb \\
    \hline
    Goal \\
    \hline
    \end{tabular}
    \vspace{1mm}
    \caption{Shortest path task features}
    \label{tab:wsp_grid-rep}
  \end{minipage}
  \begin{minipage}[b]{0.5\textwidth}
  \centering
  \begin{tabular}{|c|}
    \hline
    Agent facing North \\
    \hline
    Agent facing South \\
    \hline
    Agent facing West \\
    \hline
    Agent facing East \\
    \hline
    Start \\
    \hline
    Goal \\
    \hline
    \end{tabular}
    \vspace{1mm}
    \caption{TSP task features}
    \label{tab:tsp_grid-rep}
    \end{minipage}
\end{table}

\paragraph{Shortest path environment.}
A task in the shortest path environment is represented by a grid-world containing the agent, goals, muds, and bombs. Each possible configuration of a grid-world, including the agent's location and direction, is associated with a state in the shortest path environment MDP, $\mathcal{M}_\mathrm{path}$.
The size of the state space sees a combinatorial growth with the size of the grid, corresponding to different ways of placing bombs/muds/goals. Hence, the state-space is intractably large to enumerate. The agent's action space consists of 3 actions, $\mathcal{A} = \{\texttt{move}, \texttt{left}, \texttt{right}\}$. The actions \texttt{left} or \texttt{right} changes the agent's direction accordingly. The agent moves one step forward in its current direction with the action \texttt{move}. The environment reward function $R^{E_{\mathrm{path}}}$ has a $-1$ reward value for each action performed by the agent. Reaching a goal cell has a $+10$ reward value. There is a reward value of $-1$ for encountering a cell with mud and a reward value of $-5$ for encountering a bomb.
Reaching a goal or a bomb ends the agent's episode. 

Each state $s$ is characterized by a feature mapping $\phi^{E_{\mathrm{path}}}(s)$ which encodes the agent's location and direction, as well as the position of bombs, muds, and goals in the grid-world. In our environment, we consider grid-worlds of size $6\times6$, and each cell in the grid has a binary feature vector of length $7$ as shown in Table \ref{tab:wsp_grid-rep}. The first $4$ features are a one-hot encoding representation of the agent's location and direction in the grid-world. The last 3 binary features represent the presence or absence of either a mud, bomb, or goal respectively at a cell. Consequently, the feature mapping
$\phi^{E_{\mathrm{path}}}(s)$ is of dimension $6\times6\times7$.

\paragraph{TSP environment.}
A task in the TSP environment is represented by a grid-world containing the agent and goal cells. Each possible configuration of a grid-world is associated with a state in the TSP environment MDP, $\mathcal{M}_{\mathrm{tour}}$, similar to $\mathcal{M}_{\mathrm{path}}$. The agent's action space $\mathcal{A} = \{\texttt{move}, \texttt{left}, \texttt{right}\}$ is defined the same as for $\mathcal{M}_{\mathrm{path}}$.
In this environment, the reward function $R^{E_{\mathrm{tour}}}$ has a $+10$ reward value for completion of a successful tour, i.e., arriving back at the initial location after having visited all the goals in the grid-world. Similar to the shortest path environment, there is a reward value of $-1$ for each action performed by the agent. The agent's episode ends on the completion of a successful tour or after a certain time horizon.

Each state $s$ in the TSP environment is characterized by a feature mapping $\phi^{E_{\mathrm{tour}}}(s)$ similar to the shortest path environment. We again consider grid-worlds of size $6\times6$, and each cell in the grid has a binary feature vector of length $6$ as shown in Table \ref{tab:tsp_grid-rep}. The first $4$ features are a one-hot encoding representation of the agent's location and direction in the grid-world. The next feature captures the starting cell of the agent, which signals the final point of the tour. The last binary feature represents the presence or absence of a goal at a given cell.
Hence, the feature mapping $\phi^{E_{\mathrm{tour}}}(s)$ is of dimension $6\times6\times6$.

\subsection{Dataset Generation}

Here, we outline the dataset generation process. We create separate training, validation, and test sets for both of our navigation environments. Further, optimal paths were computed for all the tasks in the training set for both environments. These are provided as demonstrations to the learner during the training phase. In the case of multiple optimal paths for a task, each optimal path was included as a unique demonstration.

\paragraph{Shortest path environment.}
For the shortest path navigation tasks, we sample grid-worlds containing several muds and bombs, both in the range $\{0, \dots, 12\}$. The agent's initial position and location of goals, muds, and bombs are all sampled at random without overlap. The training, validation, and test sets contain $100$, $10$, $30$ grid-worlds respectively for each combination of muds and bombs, leading to datasets of sizes $16900$, $1690$, and $5070$ respectively. Additionally, each dataset contains an equal percentage of grid-worlds with a single goal cell and with two goal cells. 

\paragraph{TSP environment.}

For the TSP navigation tasks, we sample grid-worlds containing goal cells in the range $\{2,\dots,4\}$. The agent's initial position and location of goals are sampled at random without overlap. The training, validation, and test sets contain $2000$, $100$, $500$ grid-worlds respectively for each unique number of goal cells in a task, leading to datasets of size $6000$, $1500$, and $300$ respectively.


\subsection{Teacher's Difficulty Score}

As explained in Sections \ref{sec:curr-teacher-algo} and \ref{sec:experiment_without_teacher}, for the learner-centric setting we define the teacher's difficulty $\Psi^E(\xi)$ using a task-specific difficulty score.

\paragraph{Shortest path environment.}
For the shortest path tasks we define the following difficulty score:
\begin{equation}
\Psi^E(\xi) = \frac{\#goals \times \#optimal\_paths}{optimal\_reward}.
\label{eq:wsp_domain_diff}
\end{equation}
Intuitively, the difficulty score in Eq.~(\ref{eq:wsp_domain_diff}) is proportional to the difficulty of a task as the greater the number of goals present and optimal paths, the more challenging the task is for the learner. Additionally, a higher optimal reward implies a shorter path to a goal that is less challenging for the learner.

\paragraph{TSP environment.}
For the TSP tasks we define the teacher's difficulty score as:
\begin{equation}
\Psi^E(\xi) = \frac{\#goals}{optimal\_reward - greedy\_gap},
\label{eq:tsp_domain_diff}
\end{equation}
where \emph{greedy gap} is defined as the difference in reward between the optimal tour and the greedy tour for the given task. In the greedy tour, the agent repeatedly navigates to the closest goal which has not been visited yet. Once all goals have been visited, the agent returns to its initial location. The greedy tour is not necessarily the optimal tour for a task. 

Following a similar intuition as before, we see that the difficulty score of Eq.~(\ref{eq:tsp_domain_diff}) is proportional to the difficulty of a task. The greater the number of goals and the lower the optimal reward, the greater the difficulty of the task for the learner. Further, tasks with a larger $greedy\_gap$ are more complex for the learner.

In both Eqs.~(\ref{eq:wsp_domain_diff}) and (\ref{eq:tsp_domain_diff}) the denominator for the training tasks are linearly transformed to make all values $\geq 1$.

\subsection{Scheduling Mechanism}

\begin{algorithm}[t]
    \caption{Scheduling Mechanism}
    \begin{algorithmic}[1]
        \State \textbf{Initialization:} parameters a, b and total training epochs $N$.
        \For{Epoch $e = 1,2,\dots,N$}
            \State Curriculum strategy computes a preference over all demonstrations $\Xi$.
            \State Scheduling size is computed as $X = \begin{cases} b|\Xi| + \frac{e}{aN}(1-b)|\Xi| & \text{if } e < aN \\
            |\Xi| & \text{otherwise}
            \end{cases}$
            \State The $X$ most preferred demonstrations are provided to the learner in random batches.
        \EndFor{}
    \end{algorithmic}
    \label{alg:scheduling.shortest.path}
\end{algorithm}

As commonly done in prior work \cite{weinshall2018curriculum,wu2020curricula} when training neural networks using curriculum learning, we incorporate randomization in the training process for our \cur~algorithm and its variants using a scheduling mechanism. Demonstrations of higher preference are prioritized at the beginning of training, while during later stages, all demonstrations are provided with uniform probability to the learner.

In our experiments, we use a linear scheduling mechanism \cite{wu2020curricula}, where the first training epoch includes a fraction $b$ of the total demonstrations. The number of demonstrations included grows linearly every subsequent epoch such that by the time a fraction $a$ of the total epochs are completed, all the demonstrations are included. Algorithm \ref{alg:scheduling.shortest.path} details the scheduling mechanism. The demonstrations in an epoch are provided to the learner in randomly ordered batches. In our experiments we set $a = 0.8$ and $b = 0.5$.

\subsection{Learner model}

\paragraph{Training hyperparameters.}
\looseness-1For both navigation environments, the learners were trained for 40 epochs with an initial learning rate of 0.01 and a batch size of 32 demonstrations. The learning rate was decayed by a factor of 0.5 after every 500 batches of demonstrations.
The learning rate decay rule ensures the learning rate is consistent across the different curriculum algorithms since \cur~and its variants utilize a different number of training tasks in each epoch due to the scheduling mechanism. Our models were trained on Nvidia Tesla V100 GPUs.

\paragraph{Network Architecture}

\begin{table*}[t]
  \centering
  \begin{tabular}{|c||c|}
    \hline
    & \textbf{Input feature mapping} $6\times6\times d$\\
    \cline{2-2}
    \multirow{2}{*}{Convolution} & Conv2D, kernel size = 3, padding = 1, $d \rightarrow$ 32\\
    & ReLU \\
    \hline
    \multirow{6}{*}{Residual Block 1} & Conv2D, kernel size = 3, padding 1, 32 $\rightarrow$ 32\\
 & ReLU\\
 & Conv2D, kernel size = 3, padding 1, 32 $\rightarrow$ 32\\
 & ReLU\\
 & Conv2D, kernel size = 3, padding 1, 32 $\rightarrow$ 32\\
 & ReLU\\
    \hline
    \multirow{6}{*}{Residual Block 2} & Conv2D, kernel size = 3, padding 1, 32 $\rightarrow$ 32\\
 & ReLU\\
 & Conv2D, kernel size = 3, padding 1, 32 $\rightarrow$ 32\\
 & ReLU\\
 & Conv2D, kernel size = 3, padding 1, 32 $\rightarrow$ 32\\
 & ReLU\\
    \hline
    Fully Connected & Linear, $6\times 6 \times 32$ $\rightarrow$ $512$\\
    \hline
    Fully Connected & Linear, $512$ $\rightarrow$ $256$\\
    \hline
    Fully Connected & Linear, $256$ $\rightarrow$ $3$\\
    \hline
  \end{tabular}
  \vspace{1mm}
  \caption{Network architecture}
  \label{tab:task_network_arch}
\end{table*}

The learner's neural network takes as input the feature mapping $\phi^E(s)$ of a state $s$. The dimension of the feature mapping is given by $6\times6\times d$, where $d=7$ for states in the shortest path navigation environment and $d=6$ for states in the TSP navigation environment. In turn, the learner's neural network outputs a vector of size 3, which provides a probability distribution over actions after the softmax function is applied. The architecture of the neural network is provided in Table \ref{tab:task_network_arch}.

\begin{figure*}[t]
\begin{minipage}{0.49\textwidth}    
        \centering
        \includegraphics[width=0.6\textwidth]{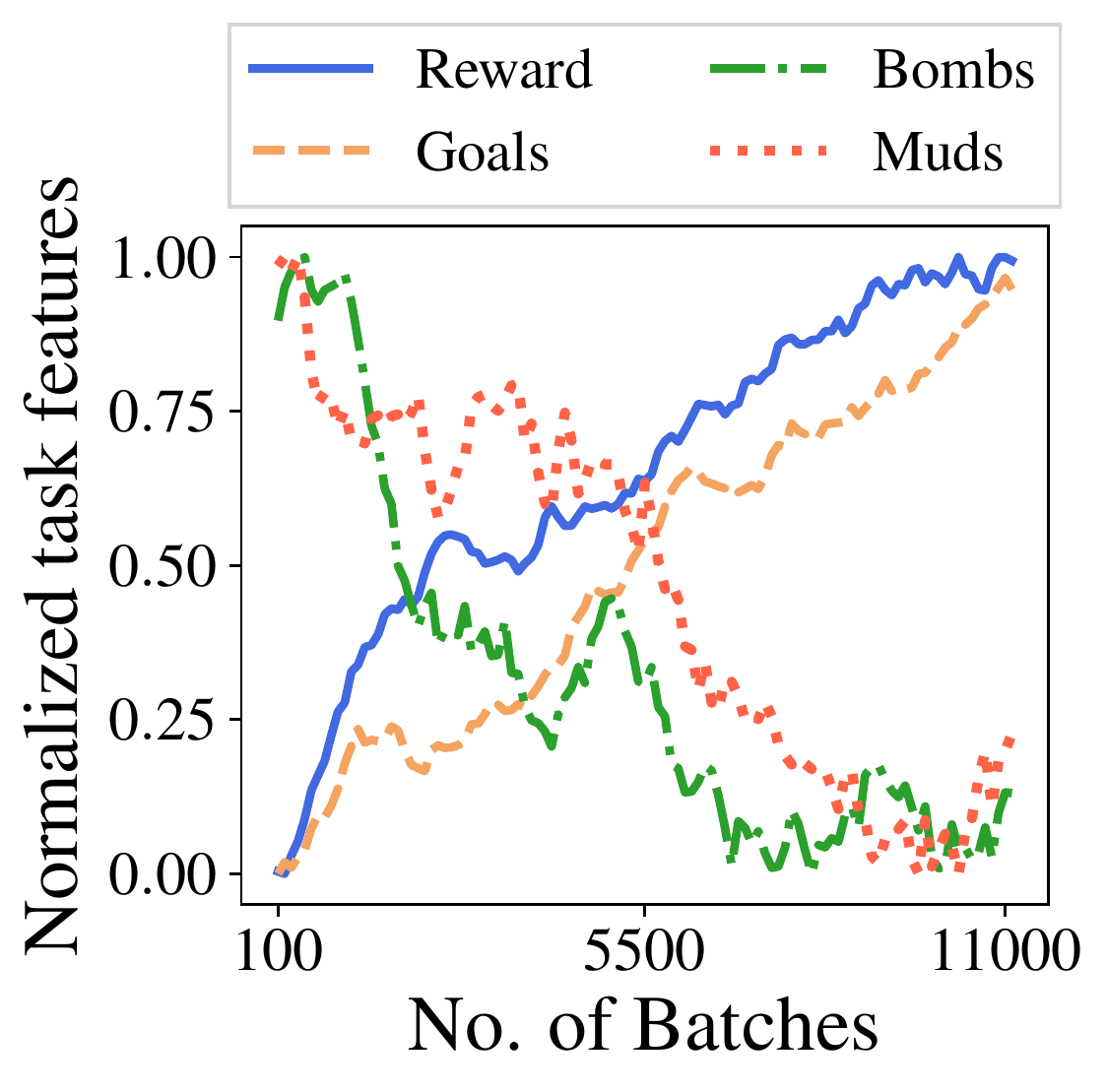}
    \caption{Shortest path environment curriculum visualization.}
    \label{fig:wsp_curriculum}
    \end{minipage}
    \qquad
    \begin{minipage}{0.49\textwidth}
        \centering
        \includegraphics[width=0.6\textwidth]{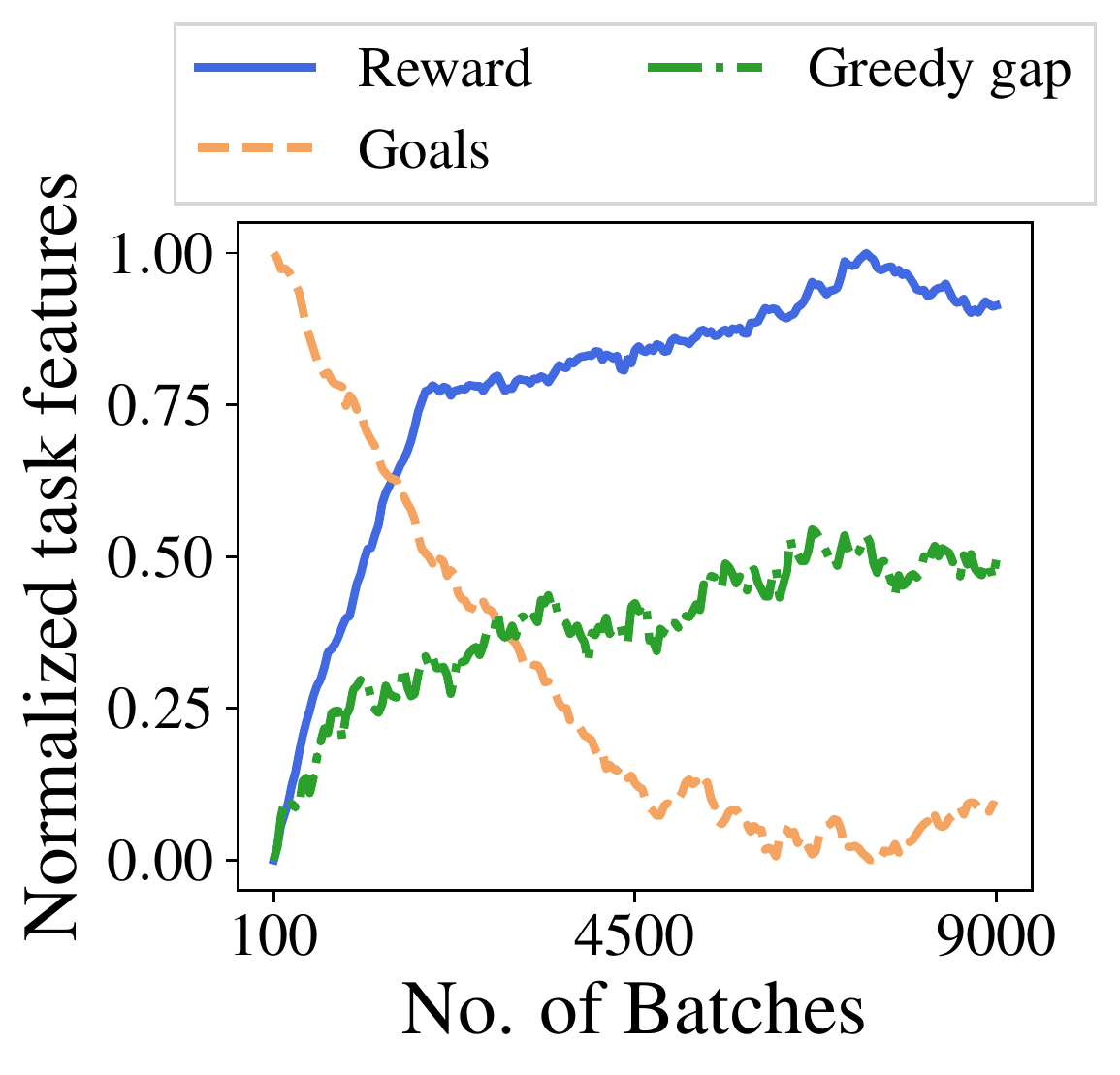}
    \caption{TSP environment curriculum visualization.}
    \label{fig:tsp_curriculum}
    \end{minipage}
\end{figure*}

\subsection{Curriculum Visualization}

In addition to the results presented in Section \ref{sec:experiment_without_teacher}, we visualize the curriculum generated by our \cur~algorithm for the shortest path and TSP environments in Figs.~\ref{fig:wsp_curriculum} and \ref{fig:tsp_curriculum} respectively. In Figs.~\ref{fig:wsp_curriculum} and \ref{fig:tsp_curriculum}, the y-axis represents different features of the tasks provided to the learner, normalized in the range $[0,1]$ and calculated as a moving average over the previous $100$ batches. The x-axis represents the number of demonstrations provided to the learner.

\paragraph{Shortest path environment.}
Fig.~\ref{fig:wsp_curriculum} shows that at the beginning of training, the \cur~algorithm picks tasks with a fewer goals and a higher number of muds/bombs. We hypothesize that this teaches the agent how to avoid muds and bombs while navigating to a goal. During later stages of training, \cur~picks tasks with higher optimal rewards and a greater number of goals. Here we believe the agent is taught how to identify the path with maximum reward among all paths that lead to a goal. Essentially the learner is first taught the general navigation task followed by the most difficult concept of deciding the optimal path.

\paragraph{TSP environment.}

Fig.~\ref{fig:tsp_curriculum} illustrates that at the beginning of training \cur~selects tasks with a greater number of goals, but with a low greedy gap. This would teach the learner the general navigation problem of visiting all goals. As training progresses, \cur~picks tasks with a greater greedy gap. We hypothesize that these tasks teach the learner the most difficult concept of planning the optimal tour.